\documentclass{article}
\usepackage{arxiv}
\renewcommand{\shorttitle}{GRAPE for Query-Efficient High-Dimensional Bayesian Optimization}
\renewcommand{\headeright}{}
\renewcommand{\undertitle}{}
\usepackage[utf8]{inputenc}
\usepackage[T1]{fontenc}
\usepackage[hyphens]{url}
\usepackage{graphicx}
\usepackage{natbib}
\usepackage{caption}
\usepackage{algorithm}
\usepackage{algorithmic}
\usepackage{booktabs}
\usepackage{multirow}

\usepackage{amsmath,amssymb,amsfonts,amsthm}
\usepackage{xspace}
\usepackage{microtype}
\usepackage{enumitem}
\usepackage{tikz}
\usetikzlibrary{arrows.meta,calc,positioning,decorations.pathmorphing}
\allowdisplaybreaks

\definecolor{gradblue}{HTML}{4C72B0}
\definecolor{progred}{HTML}{C44E52}
\definecolor{grapepurple}{HTML}{9467BD}

\newtheorem{theorem}{Theorem}

\newtheorem{proposition}{Proposition}

\newtheorem{definition}{Definition}
\newtheorem{assumption}{Assumption}

\newtheorem*{theorem*}{Theorem}
\newtheorem*{proposition*}{Proposition}

\newcommand{\bx}{\mathbf{x}}
\newcommand{\bv}{\mathbf{v}}
\newcommand{\bz}{\mathbf{z}}
\newcommand{\by}{\mathbf{y}}
\newcommand{\bX}{\mathbf{X}}
\newcommand{\bI}{\mathbf{I}}
\newcommand{\bmu}{\boldsymbol{\mu}}
\newcommand{\bSigma}{\boldsymbol{\Sigma}}
\newcommand{\calD}{\mathcal{D}}
\newcommand{\calX}{\mathcal{X}}
\newcommand{\calH}{\mathcal{H}}
\newcommand{\calN}{\mathcal{N}}

\DeclareMathOperator*{\argmin}{\mathrm{arg\,min}}
\DeclareMathOperator*{\argmax}{\mathrm{arg\,max}}
\DeclareMathOperator{\Tr}{\mathrm{Tr}}

\newcommand{\method}{\textsc{GRAPE}\xspace}
\newcommand{\mpd}{\textsc{MPD}\xspace}
\newcommand{\gibo}{\textsc{GIBO}\xspace}
\newcommand{\turbo}{\textsc{TuRBO}\xspace}
\newcommand{\vbo}{\textsc{VBO}\xspace}
\newcommand{\dlogei}{\textsc{D-LogEI}\xspace}
\newcommand{\minucb}{\textsc{MinUCB}\xspace}
\newcommand{\nestbo}{\textsc{NeST-BO}\xspace}

\title{\method: Gradient Refinement and Progress-Aware Exploitation for Query-Efficient High-Dimensional Bayesian Optimization}

\author{
  Richard Cornelius Suwandi \\
  School of Artificial Intelligence \\
  The Chinese University of Hong Kong, Shenzhen \\
  Shenzhen, China \\
  \texttt{richardsuwandi@link.cuhk.edu.cn} \\
  \And
  Feng Yin\thanks{Corresponding author.} \\
  School of Artificial Intelligence \\
  The Chinese University of Hong Kong, Shenzhen \\
  Shenzhen, China \\
  \texttt{yinfeng@cuhk.edu.cn} \\
}

\date{}

\begin{document}

\maketitle

{\let\thefootnote\relax\footnotetext{Preprint. Under review.}

\begin{abstract}
Optimizing expensive, high-dimensional black-box functions remains a central challenge in modern machine learning and scientific discovery. While local Bayesian optimization mitigates the curse of dimensionality, existing techniques often prioritize the probability of descent over the magnitude of progress. This leads to overly conservative steps that yield negligible improvement, wasting queries on directions that are nearly certain to descend but offer little decrease. We introduce Gradient Refinement and Progress-Aware Exploitation (GRAPE), a two-stage framework that first sharpens the local gradient posterior via a closed-form acquisition function, then selects update directions by maximizing the expected decrease conditional on descent. Theoretical analysis proves that this gradient refinement stage monotonically minimizes local uncertainty and that the progress-aware direction converges to true steepest descent as the posterior sharpens. Empirically, GRAPE demonstrates superior query efficiency across high-dimensional tasks: in black-box adversarial attacks, it achieves an average $5.4\times$ speedup over baselines, and on large language model prompt optimization tasks, it outperforms the second best method by a reduction of $3.8$ log-units in the final average regret.
\end{abstract}


\section{Introduction}
\label{sec:intro}

Bayesian optimization (BO) is a sample-efficient framework for optimizing
expensive black-box objectives whose gradients are unavailable and whose
evaluations may be noisy~\citep{shahriari2016taking, frazier2018tutorial,
garnett2023bayesoptbook}. It maintains a probabilistic surrogate of the
objective, typically using a Gaussian process (GP), and selects successive queries by
maximizing an acquisition function that balances exploration and
exploitation~\citep{jones1998efficient, mockus1978application}. This design
makes BO particularly attractive whenever each evaluation is costly, as fewer
queries translate directly into reduced experimental or computational expense.
Consequently, BO has been applied across a broad range of domains, including
hyperparameter tuning~\citep{snoek2012practical, arango2021hpo}, neural
architecture search~\citep{white2021bananas}, safe controller
configuration~\citep{berkenkamp2016safe}, chemical reaction and materials
design~\citep{shields2021bayesian, wang2022bayesian}, black-box adversarial
attacks~\citep{ru2020bayesopt, shu2023zeroth}, and more recently, prompt optimization for
large language models~\citep{zhou2022large, chew2026bolt}. In each of these settings, the underlying objective is expensive to evaluate, whether because it requires training a neural network, running a physical experiment, or querying a large language model, making query efficiency the primary performance metric.

Despite this sample efficiency, standard BO often degrades as the input dimension
grows. A limited evaluation budget leaves
most of the domain unexplored, a surrogate becomes uncertain over wide
regions, and optimizing a global acquisition function becomes increasingly
difficult \citep{doumont2026highdimBO}. The fundamental issue is that the volume of the search space grows exponentially with dimension, so a fixed evaluation budget covers a vanishingly small fraction of the domain. To solve these challenges, recent works have shown that an appropriate choice of GP kernel function and its hyperparameters can make standard BO
surprisingly competitive in some high-dimensional problems
\citep{hvarfner2024vanilla,xu2025standard, suwandi2025adaptive}, although this behavior is sensitive to
problem structure and GP hyperparameter initialization \citep{papenmeier2025understanding}.
Existing high-dimensional BO methods commonly address this problem by
assuming low-dimensional structures or restricting
search to local regions. Yet, structural assumptions often fail when the objective
is poorly understood, whereas local strategies trade global guarantees for
sample-efficient improvement near a promising evaluation point.

Gradient-aware local BO offers a particularly direct local strategy by leveraging a key property of Gaussian processes (GPs): a belief over the objective naturally induces a joint Gaussian belief over its gradient \citep{williams2006gaussian}. As a result, even in the absence of explicit gradient observations, function-value evaluations can be used to refine the posterior over gradients. This is especially valuable in high dimensions, where the cost of exploring the full input space is prohibitive but local gradient information can guide efficient descent. For example, \gibo \citep{muller_local_2021} first reduces gradient uncertainty before proceeding along the posterior mean gradient, while \mpd \citep{nguyen2022local} selects the direction with the highest posterior probability of descent. In contrast, \minucb \citep{fan2024minimizing} exploits a local upper confidence bound. These approaches differ in how they leverage posterior information to guide local steps, but all share the insight that refining the gradient belief before moving can yield more productive queries than global exploration.

In this work, we specifically revisit the notion of descent probability. While descent probability quantifies the certainty that a given direction descends, it does not capture the magnitude of improvement when descent occurs. A direction may be almost certain to descend while yielding only negligible progress, whereas another, less certain direction could result in substantially greater improvement when successful. Simply maximizing descent probability therefore favors safe, but potentially unproductive directions and can overlook opportunities for more significant progress. We address this limitation with Gradient Refinement and Progress-Aware
Exploitation (\method). Our main contributions are:
\begin{itemize}
\item We propose a two-stage local BO framework that first sharpens the local
gradient posterior via a closed-form acquisition, then selects update
directions by maximizing the expected decrease conditional on descent.
\item We prove that gradient refinement monotonically minimizes local
uncertainty and that the progress-aware direction converges to true
steepest descent as the posterior sharpens.
\item Empirically, \method achieves an average $5.4\times$ speedup on
black-box adversarial attacks and outperforms the second-best method by
$3.8$ log-units on LLM prompt optimization.
\end{itemize}

\section{Related Work}
\label{sec:related}

\subsection{High-Dimensional Bayesian Optimization}

High-dimensional BO methods commonly reduce the effective search complexity
through structural assumptions or local modeling. Random-embedding methods
optimize in a low-dimensional subspace under the assumption that the objective
has low intrinsic dimension \citep{wang2016rembo}. Sparse axis-aligned models
instead learn a small set of influential coordinates
\citep{eriksson2021saasbo}. These approaches can be effective when their
structural assumptions match the objective, but their performance can
deteriorate when the active structure is unknown or misspecified.

Local modeling offers an alternative that does not require an explicit
low-dimensional structure. \turbo \citep{eriksson2019scalable} fits GPs
within adaptive trust regions that expand after successful steps and contract
after failures, concentrating evaluations near promising points while using
restarts to escape from poor basins. This local emphasis is shared by the
gradient-aware methods discussed next, which replace trust-region
success or failure rules with guidance from the GP derivative posterior.

Recent work has also revisited standard, or ``vanilla'' BO in high
dimensions. \citet{hvarfner2024vanilla} scale a log-normal prior on GP
length scales with dimension and obtain a strong global \dlogei baseline.
\citet{xu2025standard} show that poor length-scale initialization can produce
vanishing training gradients for squared-exponential kernels, while
Mat\'ern kernels or dimension-aware initialization can avoid this failure in some cases. Most recently,
\citet{suwandi2025adaptive} move beyond fixed kernel choices with CAKE, which
uses LLMs to evolve GP kernel structures from observed data and ranks the
resulting candidates by both model fit and acquisition utility.

\subsection{Gradient-Aware Local Bayesian Optimization}

More closely related to our work are gradient-aware local methods, which alternate between learning about the gradient and
using that information to update the current point. \gibo
\citep{muller_local_2021} minimizes posterior gradient uncertainty before
moving along the negative posterior mean gradient. \mpd
\citep{nguyen2022local} observes that the posterior mean gradient need not
maximize descent probability and derives the most probable descent direction.
\minucb \citep{fan2024minimizing} replaces the gradient step with local
minimization of an upper confidence bound. \nestbo
\citep{tang2025nestbo} jointly predicts the gradient and Hessian from the GP
and selects evaluations that reduce uncertainty in a modified Newton step.
This curvature information can improve local scaling and convergence in
ill-conditioned regions, but learning a $d\times d$ Hessian introduces
quadratic derivative complexity. The convergence behavior of gradient-aware
local methods has also been studied under smooth-kernel assumptions
\citep{wu2023behavior}.

\method modifies the exploitation criterion rather than increasing the
derivative order or optimizing a confidence bound. Its progress-aware
criterion ranks directions by the expected decrease conditional on descent,
while gradient refinement greedily reduces the uncertainty of the gradient
used by that criterion. The approach requires only function evaluations and
first-order GP derivatives, making it applicable in settings where second-order information is too costly to recover or where the objective is not sufficiently smooth for Hessian-based methods to be reliable.

\subsection{Global--Local Hybrids and Probabilistic Numerics}

Local optimization can also be embedded within a global optimization strategy.
For example, BLOSSOM switches between BO acquisitions and conventional local optimization
\citep{mcleod2018optimization}, while TREGO alternates global BO steps with
trust-region steps \citep{diouane2023trego}. A gradient-aware optimizer such as
\method can serve as the local component of a similar hybrid when global
coverage is required. More broadly, our approach adopts the probabilistic numerics view that
numerical computation can be formulated as inference
\citep{hennig2022probabilistic}. The unavailable gradient is treated as a
latent quantity inferred from function evaluations, and posterior uncertainty
directly determines which information to collect and how to move. This perspective unifies the exploration and exploitation stages of \method{} under a single Bayesian framework, where both the choice of refinement queries and the selection of descent directions are governed by the same posterior distribution. 

\section{Preliminaries}
\label{sec:background}

\subsection{Problem Setting}

Let $f:\calX\to\mathbb{R}$ be a continuous black-box function on a compact
domain $\calX\subset\mathbb{R}^d$. Given a starting point $\bx_0$, our goal
is to minimize $f$ in a local neighborhood
$\calX(\bx_0)\subseteq\calX$ of $\bx_0$:
\begin{equation}
\bx^* 
=\argmin_{\bx\in\calX(\bx_0)} f(\bx).
\end{equation}
We can query the objective at locations of our choice and observe
\begin{equation}
y=f(\bx)+\varepsilon,\qquad \varepsilon\sim\calN(0,\sigma^2),
\end{equation}
where $\sigma^2\geq 0$ is the observation-noise variance.
Here, the gradient of $f$ is unavailable, and each function evaluation is assumed
to be costly. Consequently, an effective black-box optimizer must decide both where to
query for information and how to use that information to move from the
current point. Local optimization makes this problem more tractable in high
dimensions by replacing broad coverage of $\calX$ with targeted learning
near the current iterate.

\subsection{Bayesian Optimization with Gaussian Processes}

Bayesian optimization (BO) tackles expensive black-box optimization by maintaining a probabilistic model of the unknown objective and using that model to select informative evaluation points \citep{garnett2023bayesoptbook}. Specifically, we place a Gaussian process (GP) prior $f\sim\mathcal{GP}(\mu,k)$ on $f$, where $\mu$ is the mean function and $k$ is the kernel (covariance) function. For any finite set of inputs, the corresponding function values are jointly Gaussian \cite{williams2006gaussian}, allowing us to compute an analytic posterior after observing the data $\calD=(\bX,\by)$. Define $\mathcal{K} = k(\bX, \bX) + \sigma^2\bI$. Here, $k(\bx,\bX)$ is the row vector of covariances between $\bx$ and the training inputs $\bX$, and $k(\bX,\bx')$ is the corresponding column vector. The posterior mean and covariance are given by
\begin{align}
\mu_\calD(\bx)
&=\mu(\bx)+k(\bx,\bX)\mathcal{K}^{-1}(\by-\mu(\bX)),\\
k_\calD(\bx,\bx')
&=k(\bx,\bx')-k(\bx,\bX)\mathcal{K}^{-1}k(\bX,\bx').
\end{align}
The posterior mean $\mu_\calD$ provides a prediction of the objective at $\bx$, while the posterior variance $k_\calD(\bx,\bx)$ quantifies the uncertainty. BO uses these predictive quantities to construct an acquisition function, which balances exploration and exploitation by assigning high utility to points that are either promising (low predicted objective) or uncertain (high variance). As an example, expected improvement \citep{jones1998efficient,mockus1978application} favors points that are likely to yield an improvement over the current best. After evaluating the acquisition-maximizing point, BO updates the GP model with the new observation and repeats the process.

\subsection{Gradient-Aware Local Optimization}

GPs are particularly suitable for gradient-aware local optimization because
linear operations preserve Gaussianity. If $\mu$ is differentiable and $k$
is twice differentiable, function values and partial derivatives are jointly
Gaussian \citep{williams2006gaussian}. Thus, although gradients are never
queried directly, differentiating the posterior GP with respect to $\bx$ yields
\begin{equation}
\label{eq:posterior}
\nabla f(\bx)\mid\calD\sim\calN(\bmu_\bx,\bSigma_\bx),
\end{equation}
where
\begin{align}
\bmu_\bx
&=\nabla\mu(\bx)+\nabla k(\bx,\bX)
  \mathcal{K}^{-1}(\by-\mu(\bX)),\\
\bSigma_\bx
&=\nabla k(\bx,\bx)\nabla^\top
 -\nabla k(\bx,\bX)\mathcal{K}^{-1}
  k(\bX,\bx)\nabla^\top.
\end{align}
The differential operator before $k$ acts on its first argument and the operator
after $k$ acts on its second argument. The
vector $\bmu_\bx$ is the posterior mean of the local gradient, and
$\bSigma_\bx$ quantifies the uncertainty in that gradient.


The posterior in Eq.~\eqref{eq:posterior} induces a general two-stage local BO
strategy at the current point $\bx$: Stage~1 selects function evaluations that
refine the local derivative belief, and Stage~2 uses that belief to choose a
descent step. Several methods follow this general framework. \gibo minimizes total gradient uncertainty and takes a step along the posterior mean of the gradient \citep{muller_local_2021}. \minucb retains a similar exploration stage, but instead of stepping along the mean gradient, it exploits by minimizing a local upper confidence bound \citep{fan2024minimizing}. \nestbo extends this idea by jointly learning gradient and Hessian posteriors to target a modified Newton step \citep{tang2025nestbo}. Most closely related to our approach, \mpd leverages the full first-order posterior to identify the direction most likely to decrease the objective \citep{nguyen2022local}. Because our exploitation criterion is directly based on this construction, we review it next.

To formalize \mpd's criterion, let $\bv\in\mathbb{R}^d$ with $\|\bv\|=1$ denote a
candidate search direction on the unit sphere at $\bx$. The corresponding
\emph{directional derivative}
\begin{equation}
\nabla_\bv f(\bx)
\;=\;
\bv^\top\nabla f(\bx)
\end{equation}
is the rate of change of $f$ at $\bx$ when moving along $\bv$. Because
$\nabla f(\bx)\mid\calD\sim\calN(\bmu_\bx,\bSigma_\bx)$, the projected
scalar is univariate Gaussian,
\begin{equation}
\begin{aligned}
\nabla_\bv f(\bx)\mid\calD
&\sim\calN(\mu_\bv,\sigma_\bv^2),\\
\mu_\bv&=\bv^\top\bmu_\bx,\qquad
\sigma_\bv^2=\bv^\top\bSigma_\bx\bv.
\end{aligned}
\end{equation}
A direction is locally descending when this directional derivative is
negative. Its posterior probability of descent is therefore
\begin{equation}
\Pr(\nabla_\bv f(\bx)<0\mid\calD)
=\Phi\!\left(-\frac{\mu_\bv}{\sigma_\bv}\right),
\label{eq:descent_probability}
\end{equation}
where $\Phi$ denotes the standard normal CDF.
Unlike a step based only on the posterior mean, this probability accounts
for both the estimated slope and its uncertainty.
\mpd maximizes this probability over directions.
When
$\bSigma_\bx$ is positive definite, its unique maximizing direction up to scale is
\begin{equation}
\bv^*_{\mpd}
=-\frac{\bSigma_\bx^{-1}\bmu_\bx}
{\|\bSigma_\bx^{-1}\bmu_\bx\|},
\end{equation}
and the corresponding maximum descent probability is
\begin{equation}
\Phi\!\left(\sqrt{\bmu_\bx^\top
\bSigma_\bx^{-1}\bmu_\bx}\right).
\end{equation}

We retain \mpd's directional posterior but argue that maximizing descent
probability alone is not enough for sample-efficient local progress. The next
section makes this limitation precise and develops the progress-aware
criterion of \method.

\section{Proposed Method}
\label{sec:method}

\subsection{Beyond Descent Probability}

The descent probability in Eq. \eqref{eq:descent_probability} depends on the
directional posterior only through the standardized mean
$\gamma_\bv=\mu_\bv/\sigma_\bv$:
\begin{equation}
\mathbb{P}\bigl(\nabla_\bv f(\bx)<0\bigr)
=
\Phi(-\gamma_\bv).
\label{eq:descent_prob_gamma}
\end{equation}
Consequently, $\Phi(-\gamma_\bv)$ is invariant to any positive rescaling
$(\mu_\bv,\sigma_\bv)\mapsto(c\mu_\bv,c\sigma_\bv)$. Two unit directions with
the same $\gamma_\bv$ therefore receive identical scores, even when one
expected slope is far steeper than the other. Descent probability thus ranks
how certain a direction is to descend, not how much decrease it offers when
descent occurs.

To see why this matters in practice, consider two candidate directions at a
point $\bx$ with directional posteriors
$(\mu_1,\sigma_1)=(-2,1)$ and $(\mu_2,\sigma_2)=(-0.2,0.1)$. Both have
$\gamma=-2$ and therefore share the same descent probability
$\Phi(2)\approx 0.977$. Yet their expected decreases conditional on descent
are $\approx 2.05$ and $\approx 0.21$, respectively, an order-of-magnitude
difference that descent probability is blind to. In high dimensions, where
the posterior is often diffuse early in optimization, such ties are common
and the resulting direction choices can waste queries on nearly flat but
certainly-descending directions. The next subsection restores this missing
magnitude.

\subsection{Progress-Aware Exploitation}
\label{sec:progress}

We rank directions by the expected decrease conditional on descent, and use
this score as the exploitation stage of \method.

\begin{definition}[Progress-aware exploitation score]
\label{def:progress}
Let $\gamma_\bv=\mu_\bv/\sigma_\bv$ denote the standardized directional
mean. The progress-aware exploitation score of direction $\bv$ at $\bx$ is
\begin{align}
\mathcal{P}(\bv) &:= \mathbb{E}\bigl[-\nabla_\bv f(\bx)\mid \nabla_\bv f(\bx)<0\bigr] \notag \\
&= \sigma_\bv\!\left[\frac{\phi(\gamma_\bv)}{\Phi(-\gamma_\bv)} - \gamma_\bv\right],
\label{eq:progress}
\end{align}
where $\phi$ and $\Phi$ are the standard normal PDF and CDF.
\end{definition}

The closed form follows directly from the mean of a truncated Gaussian.
For $Z\sim\calN(\mu,\sigma^2)$, the mean of the lower-truncated distribution
at $0$ is $\mathbb{E}[Z\mid Z<0]=\mu-\sigma\phi(\mu/\sigma)/\Phi(-\mu/\sigma)$.
Applying this identity to $Z=\nabla_\bv f(\bx)$ and negating the resulting
conditional mean yields
\begin{equation}
\mathbb{E}[-\nabla_\bv f(\bx)\mid\nabla_\bv f(\bx)<0]
= -\mu_\bv + \sigma_\bv\frac{\phi(\gamma_\bv)}{\Phi(-\gamma_\bv)},
\end{equation}
which is equal to Eq. \eqref{eq:progress}. We then select the unit direction that
maximizes this expected conditional decrease:
\begin{equation}
\bv^* = \argmax_{\|\bv\|=1}\mathcal{P}(\bv).
\label{eq:max_progress}
\end{equation}
We solve Eq. \eqref{eq:max_progress} via projected gradient ascent
\citep{calamai1987projected}. The analytic gradient of $\mathcal{P}(\bv)$ and
its relation to \mpd{}, PI, and EI are derived in Appendices~\ref{app:grad_progress}
and~\ref{app:connections}.

The progress-aware score has several desirable properties. First, it naturally balances descent probability and expected magnitude: when $\gamma_\bv\to-\infty$ (highly certain descent), $\mathcal{P}(\bv)\to -\mu_\bv$, recovering the posterior mean slope, and when $\gamma_\bv\to+\infty$ (certain ascent), $\mathcal{P}(\bv)\to 0$, correctly assigning no value to ascending directions. In the intermediate regime where $\gamma_\bv\approx 0$, the score is dominated by the uncertainty term $\sigma_\bv\phi(0)/\Phi(0)$, favoring directions with high posterior variance, a form of implicit exploration that is valuable when the gradient belief is still diffuse. Second, unlike unconditional expected descent $\mathbb{E}[-\nabla_\bv f(\bx)]=-\mu_\bv$, which can be positive even for directions that are likely to ascend (when $\mu_\bv<0$ but the variance is large), the conditional expectation ensures that we only count progress in directions that actually descend. This makes $\mathcal{P}(\bv)$ a more reliable criterion for selecting productive update directions.

\subsection{Gradient Refinement}
\label{sec:twostage}

Progress-aware exploitation ranks directions under the gradient posterior. When that posterior is
diffuse, the ranking can be unreliable: a large $\mathcal{P}(\bv)$ may reflect
posterior noise rather than a genuinely steep descent direction. This is particularly problematic early in optimization, when few evaluations have been collected and the GP surrogate is uncertain over large regions of the input space. Before moving,
\method therefore allocates a short budget of function evaluations to sharpen
$\bSigma_\bx$ at the current iterate $\bx$.

We spend this budget on $\tau_{\mathrm{explore}}$ auxiliary queries chosen to
improve the local gradient estimate at $\bx$. For a candidate location
$\bz\in\calX$, a query would return $y_\bz=f(\bz)+\varepsilon$ with
$\varepsilon\sim\calN(0,\sigma^2)$. Let
$\bSigma_{\bx|\calD\cup(\bz,y_\bz)}$ denote the gradient covariance at $\bx$
after augmenting $\calD$ with $(\bz,y_\bz)$. We select $\bz$ to maximize the
\emph{gradient-refinement acquisition}, which is the expected one-step reduction in
total gradient variance,
\begin{equation}
\begin{aligned}
\alpha_{\mathrm{ref}}(\bz)
&=
\mathbb{E}_{y_\bz}\!\Bigl[
  \Tr(\bSigma_\bx)
  -
  \Tr\!\bigl(\bSigma_{\bx|\calD\cup(\bz,y_\bz)}\bigr)
\Bigr].
\end{aligned}
\label{eq:refinement_def}
\end{equation}
Here $\Tr(\bSigma_\bx)$ is the current total gradient uncertainty, and
$\Tr\!\bigl(\bSigma_{\bx|\calD\cup(\bz,y_\bz)}\bigr)$ is the residual
uncertainty after observing $y_\bz$ at $\bz$. The expectation is over the
yet-unobserved outcome $y_\bz$ because the acquisition must be evaluated before
querying $\bz$. At first glance, Eq.~\eqref{eq:refinement_def} requires integrating over the
unknown scalar $y_\bz$. For a GP, however, the posterior covariance update
depends only on the query location, not on the realized function value. The
expectation therefore collapses, and $\alpha_{\mathrm{ref}}$ admits a closed
form.

\begin{proposition}[Closed-form gradient refinement]
\label{prop:refinement_closed}
Let $k_\calD(\cdot,\cdot)$ denote the GP posterior kernel after observing
$\calD$. Then $\alpha_{\mathrm{ref}}(\bz)$ is independent of the unobserved
value $y_\bz$ and equals
\begin{equation}
\label{eq:refinement_closed}
\alpha_{\mathrm{ref}}(\bz)
= \frac{\|\nabla_\bx k_\calD(\bx,\bz)\|^2}{k_\calD(\bz,\bz)+\sigma^2},
\end{equation}
where $\nabla_\bx k_\calD(\bx,\bz) = \partial k_\calD(\bx,\bz)/\partial\bx$
is the posterior cross-covariance gradient.
\end{proposition}
The proof is provided in Appendix~\ref{app:proof_refinement}.
Proposition~\ref{prop:refinement_closed} has two practical consequences.
First, gradient refinement can be optimized as a deterministic acquisition
function, with no Monte Carlo sampling.
Second, $\alpha_{\mathrm{ref}}(\bz)$ is large when the posterior gradient
covariance between $\bx$ and $\bz$ is strong and the posterior variance at
$\bz$ is not already negligible. Otherwise, querying $\bz$ adds little
information.


\subsection{The \method Algorithm}
\label{sec:algorithm}

We combine gradient refinement (Section~\ref{sec:twostage}) and progress-aware
exploitation (Section~\ref{sec:progress}) into \method, a two-stage local
optimization loop. Stage~1 selects auxiliary queries to sharpen the gradient
posterior at the current iterate, and Stage~2 moves along directions of maximum
expected progress under that refined posterior.

We initialize by drawing $n_{\mathrm{init}}$ random queries, forming a dataset
$\calD$, and fitting a GP. The random initialization ensures that the initial
surrogate has some coverage of the local neighborhood, preventing the gradient
posterior from being dominated by the prior. Then, for each outer iteration
$t=0,\ldots,T-1$ starting at $\bx_t$, Stage~1 sequentially selects
$\tau_{\mathrm{explore}}$ queries by maximizing the closed-form acquisition
$\bz^*\leftarrow\argmax_{\bz\in\calX}\alpha_{\mathrm{ref}}(\bz)$ in
Eq.~\eqref{eq:refinement_closed}, observes each $y=f(\bz^*)+\varepsilon$,
and refits the GP. The sequential selection allows each refinement query to
benefit from the information gained by previous queries, leading to more
efficient uncertainty reduction than batch selection. Stage~2 then repeatedly computes the unit direction of
maximum progress
$\bv^*\leftarrow\argmax_{\|\bv\|=1}\mathcal{P}(\bv)$ in
Eq.~\eqref{eq:max_progress} and takes a projected step
$\bx\leftarrow\Pi_{\calX}(\bx+\eta\bv^*)$ of size $\eta>0$, where
$\Pi_{\calX}:\mathbb{R}^d\to\calX$ is Euclidean projection onto the domain.
Each new point is evaluated and added to $\calD$ before the next direction is
chosen, so the posterior is updated after every move. This online updating
ensures that the progress-aware score reflects the most current gradient
belief, which is particularly important when the objective has varying
curvature across the local neighborhood. Exploitation stops after
$\tau_{\mathrm{exploit}}$ steps or when
$\mathcal{P}(\bv^*)<\tau_{\mathrm{thresh}}$, indicating that the posterior
predicts little conditional descent. The early-stopping criterion prevents
wasting queries on directions that are unlikely to yield meaningful progress,
which is especially valuable when the optimizer is approaching a local minimum. The outer iterate is then set to
$\bx_{t+1}\leftarrow\bx$, and after $T$ iterations \method returns the best
observed point $\argmin_{(\bx,y)\in\calD}\,y$.
Figure~\ref{fig:overview} illustrates one outer iteration of this loop.

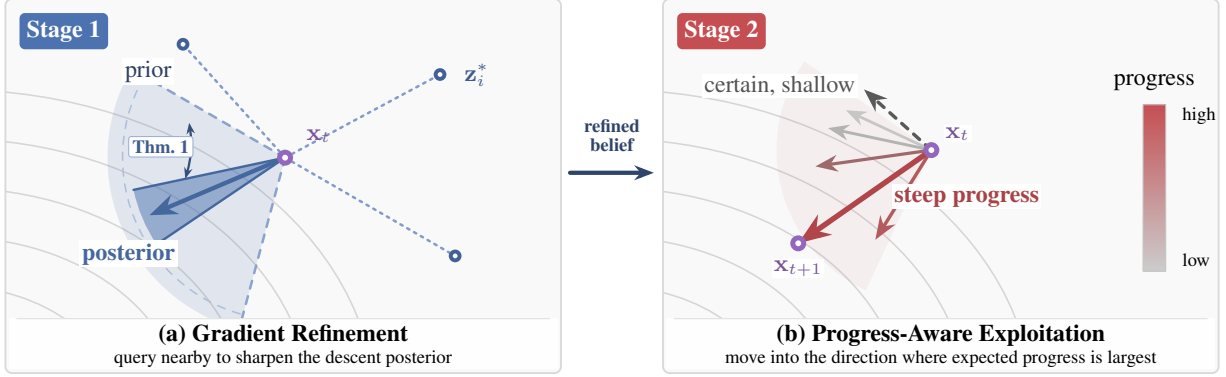
\begin{figure*}[t]
\centering
\begin{tikzpicture}[
  x=1cm, y=1cm,
  >=Stealth,
  font=\sffamily,
  line cap=round,
  line join=round,
]
  \def\panelW{7.35}
  \def\panelH{4.95}
  \def\gap{1.35}
  \def\footH{0.72}

  \fill[gray!5,rounded corners=2.5pt] (0,0) rectangle (\panelW,\panelH);
  \fill[gray!5,rounded corners=2.5pt]
    ({\panelW+\gap},0) rectangle ({2*\panelW+\gap},\panelH);
  \draw[gray!30,rounded corners=2.5pt,line width=0.7pt]
    (0,0) rectangle (\panelW,\panelH);
  \draw[gray!30,rounded corners=2.5pt,line width=0.7pt]
    ({\panelW+\gap},0) rectangle ({2*\panelW+\gap},\panelH);

  \begin{scope}
    \clip (0,\footH) rectangle (\panelW,\panelH);

    \foreach \s in {1,...,5}
      \draw[gray!32,line width=0.6pt]
        (-0.7,-0.5) ellipse ({1.35+0.85*\s} and {1.0+0.65*\s});

    \coordinate (xta) at (3.70,2.85);

    \fill[gradblue,opacity=0.14]
      (xta) -- ++(150:2.35) arc (150:255:2.35) -- cycle;
    \draw[gradblue!65,dashed,line width=0.95pt] (xta) -- ++(150:2.35);
    \draw[gradblue!65,dashed,line width=0.95pt] (xta) -- ++(255:2.35);
    \draw[gradblue!45,dashed,line width=0.6pt]
      (xta) ++(150:2.15) arc (150:255:2.15);

    \fill[gradblue,opacity=0.50]
      (xta) -- ++(192:2.05) arc (192:214:2.05) -- cycle;
    \draw[gradblue!90!black,line width=0.9pt] (xta) -- ++(192:2.05);
    \draw[gradblue!90!black,line width=0.9pt] (xta) -- ++(214:2.05);
    \draw[->,gradblue!90!black,line width=1.65pt] (xta) -- ++(203:1.95);

    \foreach \px/\py in {5.75/3.95, 5.95/1.55, 2.35/4.35} {
      \draw[gradblue!70,dotted,line width=0.95pt] (xta) -- (\px,\py);
      \fill[gradblue!85!black] (\px,\py) circle (2.4pt);
      \fill[white] (\px,\py) circle (0.9pt);
    }
    \node[text=gradblue!80!black,font=\footnotesize,anchor=west]
      at (5.95,3.95) {$\bz_i^*$};

    \fill[grapepurple] (xta) circle (3.0pt);
    \fill[white] (xta) circle (1.1pt);
    \node[text=grapepurple!80!black,font=\footnotesize,anchor=south west]
      at ($(xta)+(0.16,0.08)$) {$\bx_t$};

    \node[fill=white,inner sep=1pt,text=gradblue!55!black,font=\footnotesize,
          anchor=south] at ($(xta)+(152:2.05)$) {prior};
    \node[fill=white,inner sep=1pt,text=gradblue!90!black,
          font=\footnotesize\bfseries,anchor=east]
      at ($(xta)+(222:1.9)$) {posterior};

    \draw[{Stealth[length=1.5mm]}-{Stealth[length=1.5mm]},
          gradblue!70!black,line width=0.85pt]
      ($(xta)+(160:1.35)$) to[bend left=12] ($(xta)+(192:1.35)$);
    \node[fill=white,draw=gradblue!55,line width=0.4pt,rounded corners=1pt,
          inner sep=1.4pt,text=gradblue!80!black,font=\scriptsize\bfseries]
      at ($(xta)+(175:1.65)$) {Thm.~1};
  \end{scope}

  \node[fill=gradblue,text=white,rounded corners=2pt,inner sep=2.8pt,
        font=\bfseries\footnotesize,anchor=north west]
    at (0.18,\panelH-0.18) {Stage 1};
  \fill[white,opacity=0.92] (0.05,0.05) rectangle (\panelW-0.05,\footH);
  \draw[gray!25,line width=0.4pt] (0.15,\footH) -- (\panelW-0.15,\footH);
  \node[anchor=center,font=\small\bfseries,align=center]
    at (0.5*\panelW,0.5*\footH)
    {(a) Gradient Refinement\\[-2pt]
     {\normalfont\scriptsize query nearby to sharpen the descent posterior}};

  \draw[-{Stealth[length=2.6mm]},gradblue!60!black,line width=1.6pt]
    ({\panelW+0.12},2.65) -- ({\panelW+\gap-0.12},2.65);
  \node[fill=white,inner sep=1.5pt,text=gradblue!60!black,
        font=\scriptsize\bfseries,align=center]
    at ({\panelW+0.5*\gap},3.15) {refined\\[-1pt]belief};

  \begin{scope}[shift={({\panelW+\gap},0)}]
    \clip (0,\footH) rectangle (\panelW,\panelH);

    \foreach \s in {1,...,5}
      \draw[gray!32,line width=0.6pt]
        (-0.7,-0.5) ellipse ({1.35+0.85*\s} and {1.0+0.65*\s});

    \coordinate (xtb) at (3.55,2.95);
    \coordinate (xt1) at ($(xtb)+(215:2.15)$);

    \fill[progred,opacity=0.06]
      (xtb) -- ++(145:2.05) arc (145:245:2.05) -- cycle;

    \draw[->,gray!45,line width=1.05pt] (xtb) -- ++(155:1.25);
    \draw[->,gray!60,line width=1.1pt]  (xtb) -- ++(168:1.4);
    \draw[->,progred!45!gray,line width=1.2pt] (xtb) -- ++(188:1.55);
    \draw[->,progred!75!gray,line width=1.3pt] (xtb) -- ++(238:1.4);

    \draw[->,gray!70!black,dashed,line width=1.3pt]
      (xtb) -- ++(138:1.2) coordinate (shallow);
    \node[fill=white,inner sep=0.8pt,text=gray!65!black,font=\footnotesize,
          align=right,anchor=east]
      at ($(shallow)+(-0.08,0.05)$) {certain, shallow};

    \draw[->,progred!90!black,line width=2.05pt] (xtb) -- (xt1);
    \node[fill=white,inner sep=0.8pt,text=progred!85!black,
          font=\footnotesize\bfseries,align=left,anchor=west]
      at ($(xtb)+(215:1.15)+(0.42,0.05)$) {steep progress};

    \fill[grapepurple] (xtb) circle (3.0pt);
    \fill[white] (xtb) circle (1.1pt);
    \node[fill=white,inner sep=0.6pt,text=grapepurple!80!black,
          font=\footnotesize,anchor=south west]
      at ($(xtb)+(0.16,0.1)$) {$\bx_t$};
    \fill[grapepurple] (xt1) circle (3.0pt);
    \fill[white] (xt1) circle (1.1pt);
    \node[fill=white,inner sep=0.6pt,text=grapepurple!80!black,
          font=\footnotesize,anchor=north]
      at ($(xt1)+(0.0,-0.2)$) {$\bx_{t+1}$};

    \shade[bottom color=gray!40,top color=progred]
      (6.35,1.35) rectangle (6.65,3.55);
    \draw[gray!45,line width=0.4pt] (6.35,1.35) rectangle (6.65,3.55);
    \node[font=\footnotesize,anchor=south] at (6.5,3.62) {progress};
    \node[font=\scriptsize,anchor=west] at (6.75,3.4) {high};
    \node[font=\scriptsize,anchor=west] at (6.75,1.5) {low};
  \end{scope}

  \node[fill=progred,text=white,rounded corners=2pt,inner sep=2.8pt,
        font=\bfseries\footnotesize,anchor=north west]
    at ({\panelW+\gap+0.18},\panelH-0.18) {Stage 2};
  \fill[white,opacity=0.92]
    ({\panelW+\gap+0.05},0.05) rectangle ({2*\panelW+\gap-0.05},\footH);
  \draw[gray!25,line width=0.4pt]
    ({\panelW+\gap+0.15},\footH) -- ({2*\panelW+\gap-0.15},\footH);
  \node[anchor=center,font=\small\bfseries,align=center]
    at ({\panelW+\gap+0.5*\panelW},0.5*\footH)
    {(b) Progress-Aware Exploitation\\[-2pt]
     {\normalfont\scriptsize move into the direction where expected progress is largest}};
\end{tikzpicture}
\caption{One outer iteration of \method. \textbf{(a)}~Nearby queries
sharpen a wide prior descent belief into a narrower posterior
(Theorem~\ref{thm:refinement}). \textbf{(b)}~Candidate directions are
ranked by their expected progress.}
\label{fig:overview}
\end{figure*}

\section{Theoretical Analysis}
\label{sec:theory}

This section analyzes the two design choices introduced above: gradient
refinement and progress-aware exploitation.

\subsection{Notation and Assumptions}

Let $\calX\subset\mathbb{R}^d$ be a convex compact domain and $\calH$ a
reproducing kernel Hilbert space (RKHS) on $\calX$ with kernel $k$.
At an iterate $\bx_t$, let $\calD_t$ denote all observations collected so
far, and let $\bmu_t$ and $\bSigma_t$ denote the posterior mean and
covariance of $\nabla f(\bx_t)$ under $\calD_t$. Our analysis aims to answer two
questions. First, does a gradient-refinement observation reduce uncertainty
for a fixed iterate? Second, when that uncertainty vanishes and the posterior
mean is consistent, which direction does progress-aware exploitation select?
To carry out the analysis, we impose the following standard regularity conditions.

\begin{assumption}
\label{ass:kernel}
The kernel $k$ is stationary, four times continuously differentiable, and
positive definite.
\end{assumption}
\begin{assumption}
\label{ass:interior}
The iterates $\{\bx_t\}$ remain in the interior of $\calX$.
\end{assumption}
\begin{assumption}
\label{ass:rkhs}
The objective lies in the RKHS of $k$ with bounded norm:
$f\in\calH$ and $\|f\|_\calH\leq B$, where $B<\infty$ is a fixed bound on
the RKHS norm of $f$.
\end{assumption}

Assumptions~\ref{ass:kernel}--\ref{ass:rkhs} are standard
\citep{bull2011convergence,srinivas2012information}. Since $k$ is four times
differentiable, $f$ is twice continuously differentiable. On the compact
domain $\calX$ this implies that $\nabla f$ is Lipschitz, i.e., $f$ is
$L$-smooth for some finite constant $L>0$ \citep{wu2023behavior}.
Compactness ensures that the direction optimization and smoothness constants
are well-defined, while the interior-iterate assumption avoids boundary
projections in the limiting direction result. The RKHS assumption is a regularity condition on the objective: it requires that $f$ lies in the function space induced by the kernel, which is a common assumption in GP-based optimization analyses \citep{srinivas2012information}. While this assumption may not hold exactly in practice, it provides a useful framework for analyzing the convergence behavior of GP-based methods. In our experiments, we use the Mat\'ern-$5/2$ kernel, whose RKHS consists of functions that are twice continuously differentiable, a reasonable assumption for many practical objectives, including neural network loss landscapes and prompt accuracy surfaces.

\subsection{Gradient Uncertainty Reduction}

Building on
Proposition~\ref{prop:refinement_closed}, the result below formalizes the
one-step contraction of $\Tr(\bSigma_\bx)$ and identifies when it is strict.

\begin{theorem}[Gradient refinement reduces uncertainty]
\label{thm:refinement}
Under the joint GP model, for any candidate location $\bz$,
\[
\mathbb{E}_{y_\bz}\!\left[\Tr(\bSigma_{\bx|\calD\cup(\bz,y_\bz)})\right]
\leq \Tr(\bSigma_\bx),
\]
with equality if and only if $\nabla_\bx k_\calD(\bx,\bz)=\mathbf{0}$.
\end{theorem}

\textit{Proof.}
From Proposition~\ref{prop:refinement_closed},
$\alpha_{\mathrm{ref}}(\bz)=\|\nabla_\bx k_\calD(\bx,\bz)\|^2/(k_\calD(\bz,\bz)+\sigma^2)\geq 0$.
Hence $\mathbb{E}[\Tr(\bSigma_{\bx|\calD\cup(\bz,y_\bz)})]
=\Tr(\bSigma_\bx)-\alpha_{\mathrm{ref}}(\bz)\leq\Tr(\bSigma_\bx)$, with equality iff
$\nabla_\bx k_\calD(\bx,\bz)=\mathbf{0}$. \hfill$\square$

Since $\alpha_{\mathrm{ref}}(\bz)$ equals the expected one-step trace
reduction, selecting $\bz^*=\argmax_\bz\alpha_{\mathrm{ref}}(\bz)$ is
greedily optimal and repeated refinement monotonically decreases
$\Tr(\bSigma_\bx)$.

\subsection{Asymptotic Descent Direction}

The result below shows that, once
the gradient posterior is sharp and consistent, the progress-aware direction
converges to normalized steepest descent.

\begin{theorem}[Progress-aware exploitation approaches steepest descent]
\label{thm:progress_convergence}
Suppose $\nabla f(\bx)\neq\mathbf{0}$,
$\Tr(\bSigma_\bx)\to 0$, and
$\bmu_\bx\to\nabla f(\bx)$. Then the progress-aware direction converges to
the true steepest descent direction:
\[
\lim_{\Tr(\bSigma_\bx)\to 0}\;\bv^*
= -\frac{\nabla f(\bx)}{\|\nabla f(\bx)\|}.
\]
\end{theorem}

\textit{Proof sketch.}
As $\Tr(\bSigma_\bx)\to 0$, $\bSigma_\bx\to\mathbf{0}$ in spectral norm while
$\bmu_\bx\to\nabla f(\bx)$ by assumption. Let
$R(\gamma):=\phi(\gamma)/\Phi(-\gamma)$ denote the inverse Mills ratio and
$[a]^+:=\max(a,0)$. For any unit direction $\bv$, if
$\bv^\top\nabla f(\bx)<0$, then $\gamma_\bv\to-\infty$, $R(\gamma_\bv)\to 0$,
and $\mathcal{P}(\bv)\to-\bv^\top\nabla f(\bx)>0$. If
$\bv^\top\nabla f(\bx)>0$, then $\gamma_\bv\to+\infty$ and
$\mathcal{P}(\bv)\to 0$ by the expansion
$R(\gamma)=\gamma+\gamma^{-1}+\mathcal{O}(\gamma^{-3})$
\citep{grimmett2020probability}. Hence
$\mathcal{P}(\bv)\to[-\bv^\top\nabla f(\bx)]^+$
pointwise on $\mathbb{S}^{d-1}$. Uniform convergence on this compact set,
together with the Berge Maximum Theorem \citep{berge1963topological}, yields
$\bv^*\to -\nabla f(\bx)/\|\nabla f(\bx)\|$. A complete proof appears in
Appendix~\ref{app:proof_thm2}. \hfill$\square$


\subsection{Interpretation}

The two theorems above characterize the complementary roles of the two stages in \method{}. Theorem~\ref{thm:refinement} guarantees that each refinement query strictly reduces total gradient uncertainty unless the candidate point is already uninformative about the local gradient, a condition that holds whenever $\bz$ lies in a region where the posterior cross-covariance gradient $\nabla_\bx k_\calD(\bx,\bz)$ vanishes. This monotonic contraction ensures that the directional ranking used by the exploitation stage becomes increasingly reliable as refinement proceeds. Theorem~\ref{thm:progress_convergence} then shows that, once the posterior has sharpened sufficiently and the posterior mean is consistent with the true gradient, the progress-aware direction recovers normalized steepest descent. Together, these results establish that \method{} behaves as a principled approximation to first-order optimization: refinement drives the posterior toward a regime where the progress-aware score is well-calibrated, and exploitation then selects directions that are asymptotically equivalent to those chosen by a gradient oracle. We empirically validate Theorems~\ref{thm:refinement} and~\ref{thm:progress_convergence} in Appendix~\ref{app:emp_theory}.

\section{Experiments}
\label{sec:experiments}

In this section, we evaluate \method{} on two high-dimensional, gradient-free tasks and compare it against representative global and local BO methods. Code is available at \url{https://github.com/richardcsuwandi/grape}.

\subsection{Benchmarks}

\subsubsection{Black-Box Adversarial Attacks}
\label{sec:exp_attack}

Adversarial attacks reveal how small, visually inconspicuous input changes can
cause otherwise accurate classifiers to fail, making them an important tool
for assessing the robustness of deployed models \citep{ru2020bayesopt}. In many realistic settings,
however, the attacker can query only the model's outputs and cannot access its
parameters or gradients. Given such a black-box classifier and an image $\bz$,
we therefore seek, through function queries only, a small perturbation $\bx$
such that the perturbed image $\bz+\bx$ is misclassified. This is a natural
fit for Bayesian optimization: each query corresponds to evaluating the
classifier on a candidate perturbation, and the goal is to find a successful
attack with as few queries as possible. Following the practice of
\citet{cheng2021convergence}, we randomly select images from MNIST
\citep{lecun1998gradient} ($d=28\times28=784$) and CIFAR-10
\citep{krizhevsky2009learning} ($d=32\times32=1024$) and add a
perturbation under an $\ell_\infty$ constraint that makes the trained deep
network misclassify the image. Concretely:
\begin{itemize}[leftmargin=*,itemsep=1pt,topsep=2pt]
\item For MNIST, we use the same fully trained network as
\citet{cheng2021convergence} and adopt the constraint $\|\bx\|_\infty\le0.3$.
\item For CIFAR-10, we fully train a ResNet-18 \citep{he2016deep} using SGD
with a cosine-annealed learning rate from $0.1$ to $0$, momentum $0.9$, and
weight decay $5\times10^{-4}$ for $200$ epochs, and adopt the constraint
$\|\bx\|_\infty\le0.2$.
\end{itemize}

\subsubsection{LLM Prompt Optimization}
\label{sec:exp_llm}
The choice of
prompt can substantially affect an LLM's performance, yet designing effective
prompts by hand is labor-intensive and does not scale across models and tasks
\citep{zhou2022large}. Automatic prompt optimization is therefore valuable,
but each objective evaluation requires LLM inference, while the search space
is discrete in text space and high-dimensional in embedding space. This combination of expensive evaluations and high-dimensional discrete search makes prompt optimization a challenging testbed for black-box optimization methods. We use BoLT
\citep{chew2026bolt} because its precomputed prompt scores make evaluations
inexpensive and reproducible. It provides $5{,}014$ prompts for mathematical
reasoning, scored by their MATH-500 accuracy \citep{hendrycks2021measuring}
under Qwen3-14B \citep{yang2025qwen3}. Each prompt is embedded with
EmbeddingGemma \citep{schechtervera2025embeddinggemma} at the four Matryoshka
truncations $d\in\{128,256,512,768\}$ \citep{kusupati2022matryoshka}, and we correspondingly denote the
four tasks as PO-128 through PO-768. The search space is therefore the
discrete candidate pool $\calX_{\mathrm{PO}}=\{\bz_i\}_{i=1}^{5014}\subset\mathbb{R}^d$, where $f_{\mathrm{PO}}(\bz_i)$ is the corresponding accuracy and
the reference optimum is the highest score in the pool. Because \method
and the local baselines propose continuous steps, each proposal is
projected to its nearest embedding in the pool before evaluation. This projection introduces a quantization effect that is unique to discrete search spaces: even if the optimizer identifies a promising direction in embedding space, the actual evaluation point is constrained to the nearest available prompt, which may limit the achievable progress per step.

\subsection{Baselines}

We compare \method{} against a diverse set of baselines spanning global, local, first-order, and second-order approaches. These include random search as a sanity check, and two global BO methods: vanilla BO (\vbo) with expected improvement \citep{jones1998efficient} and \dlogei with a dimension-scaled length-scale prior \citep{hvarfner2024vanilla}. We further include the local first-order methods \mpd \citep{nguyen2022local}, \gibo \citep{muller_local_2021}, and \minucb \citep{fan2024minimizing}, as well as the trust-region method \turbo \citep{eriksson2019scalable}. Finally, we include \nestbo \citep{tang2025nestbo}, a recent second-order method that targets a modified Newton step using jointly learned gradient and Hessian posteriors. This selection covers the main design choices in high-dimensional BO: global versus local search, first-order versus second-order derivative information, and trust-region versus gradient-based step selection.

\subsection{Experimental Setup}

All GP-based methods share the same modeling protocol: an ARD
Mat\'ern-$5/2$ kernel with length scales and output scale fit by maximum
marginal likelihood, together with common initial points and evaluation
budgets. We use the Mat\'ern-$5/2$ kernel because it provides a good balance
between smoothness and flexibility: it is twice differentiable, which is
sufficient for the gradient posterior used by all first-order methods, while
avoiding the oversmoothing behavior of the squared-exponential kernel in high
dimensions \citep{xu2025standard}. For \method{}, we
set $\tau_{\mathrm{explore}}=5$, allow at most
$\tau_{\mathrm{exploit}}=30$ exploitation steps, use
step size $\eta=0.1\cdot\mathrm{diam}(\calX)$, where
$\mathrm{diam}(\calX)=\sup_{\bx,\bx'\in\calX}\|\bx-\bx'\|$ is the diameter
of the feasible domain, and early-stop threshold
$\tau_{\mathrm{thresh}}=5\times10^{-3}$, and optimize each direction with
$10$ random restarts and a warm start. On prompt optimization, we set
observation noise variance to $\sigma^2=0.001$ and run $T=200$ iterations
from five random prompts. Prompt-optimization results are averaged over $10$ random
seeds, attack query counts in Table~\ref{tab:attack} are averaged over
$10$ independent runs, while the success-rate curves in
Figure~\ref{fig:attack} are evaluated on $50$ randomly selected images per
dataset. Error bars and shaded regions denote one standard deviation.
Implementations use PyTorch \citep{paszke2019pytorch} and GPyTorch
\citep{gardner2018gpytorch} on one NVIDIA RTX~4090.
Full protocols, baseline hyperparameters, and compute details are in Appendix~\ref{app:exp_details}.

\subsection{Evaluation Metrics}

For adversarial attacks, we report the number of queries to the first
successful attack and the attack success rate under a fixed query budget
\citep{shu2023zeroth}. Fewer queries and a higher success rate are better.
The query count directly measures sample efficiency, while the success rate
curve reveals how quickly each method accumulates successful attacks as the
budget increases, a distinction that matters in practice, where the available
query budget may be constrained by time or API costs.
For prompt optimization, we report the log simple regret
\citep{chew2026bolt}, $\ell_t=\log\bigl(f(\bx^*)-f(\hat\bx_t)\bigr)$, where $f(\hat\bx_t)$ is the best score after $t$ evaluations and
$f(\bx^*)$ is the largest score in the candidate pool.
The logarithmic scale compresses the dynamic range of regret values, making it easier to distinguish methods in the low-regret regime where absolute differences are small but practically meaningful.
We assess significance with a paired two-sided Wilcoxon signed-rank test
\citep{wilcoxon1945individual}, pairing methods by run, where
improvements with $p<0.05$ are marked with $^{*}$.

\subsection{Results}

We compare \method{} with the baselines on both benchmarks.

\begin{table}[t]
\centering
\setlength{\tabcolsep}{4pt}
{\small
\begin{tabular}{lcccc}
\toprule
 & \multicolumn{2}{c}{MNIST ($d=784$)}
 & \multicolumn{2}{c}{CIFAR-10 ($d=1024$)} \\
\cmidrule(lr){2-3}\cmidrule(lr){4-5}
Method & \# of queries & Ratio
       & \# of queries & Ratio \\
\midrule
Random  & $1971\pm498^{*}$ & $9.1\times$ & $3222\pm662^{*}$ & $9.4\times$ \\
\vbo    & $1589\pm271^{*}$ & $7.4\times$ & $2563\pm436^{*}$ & $7.5\times$ \\
\dlogei & $1381\pm222^{*}$ & $6.4\times$ & $2114\pm372^{*}$ & $6.2\times$ \\
\turbo  & $1143\pm192^{*}$ & $5.3\times$ & $1239\pm191^{*}$ & $3.6\times$ \\
\gibo   & $981\pm184^{*}$ & $4.5\times$ & $1613\pm271^{*}$ & $4.7\times$ \\
\minucb & $891\pm169^{*}$ & $4.1\times$ & $1449\pm213^{*}$ & $4.2\times$ \\
\mpd    & $827\pm128^{*}$ & $3.8\times$ & $1338\pm236^{*}$ & $3.9\times$ \\
\nestbo & $719\pm136^{*}$ & $3.3\times$ & $1123\pm197^{*}$ & $3.3\times$ \\
\textbf{GRAPE} & \textbf{216$\pm$42} & \textbf{1.0$\times$} & \textbf{341$\pm$53} & \textbf{1.0$\times$} \\
\bottomrule
\end{tabular}
}
\caption{Average number of queries ($\downarrow$) $\pm$ standard deviation required for a successful attack, across 10 independent runs. Ratio is each method's query count divided by that of \method{} (higher is worse). $^{*}$ indicates that \method{} requires significantly fewer queries than the marked baseline under a paired Wilcoxon signed-rank test ($p<0.05$).}
\label{tab:attack}
\end{table}

\begin{figure}[t]
  \centering
  \includegraphics[width=\textwidth]{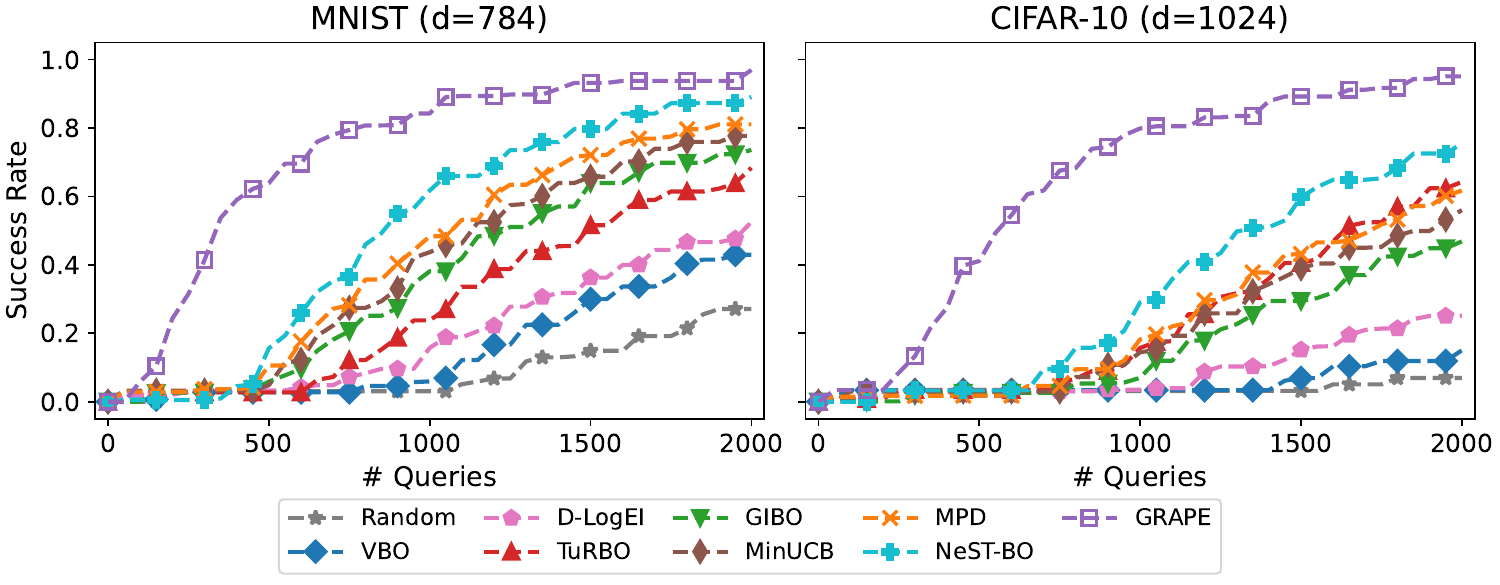}
  \caption{Attack success rate ($\uparrow$) versus number of queries on
  $50$ randomly selected images from MNIST ($d=784$) and CIFAR-10 ($d=1024$).}
  \label{fig:attack}
\end{figure}

\begin{figure*}[t]
  \centering
  \IfFileExists{figs/prompt_opt_log_regret.pdf}{%
    \includegraphics[width=\textwidth]{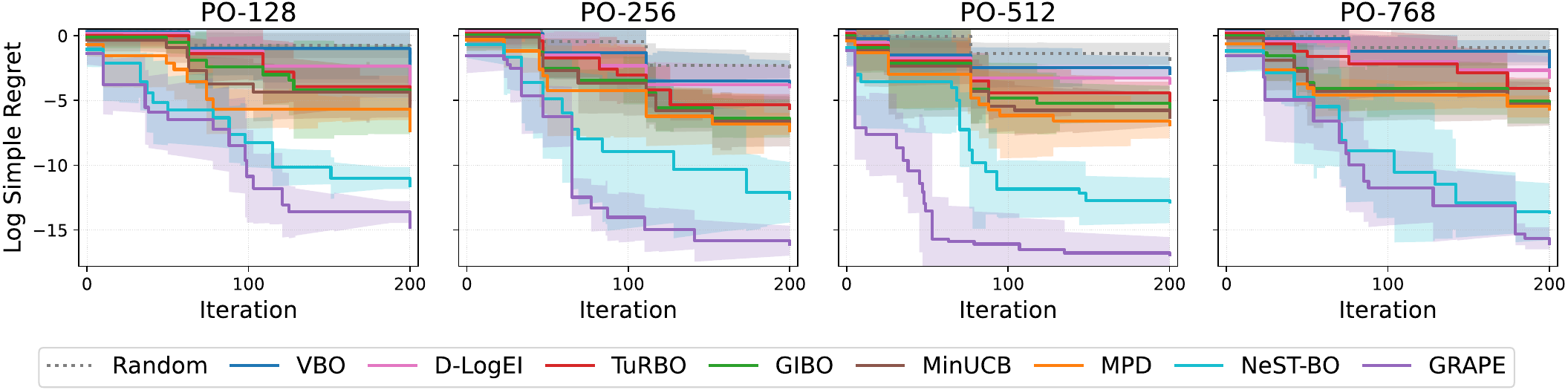}
  }{%
    \fbox{\rule{0pt}{2.8cm}\rule{\textwidth}{0pt}}%
  }
  \caption{Log simple regret ($\downarrow$) versus BO iteration on BoLT
  prompt-optimization tasks PO-128--PO-768. Shaded regions denote one
  standard deviation over 10 seeds.}
  \label{fig:prompt_opt_regret}
\end{figure*}

\subsubsection{Black-Box Adversarial Attacks}

Table~\ref{tab:attack} reports the number of queries needed for a
successful attack. In these high-dimensional settings, \method{}
substantially outperforms every baseline, reaching misclassification with
far fewer queries. All pairwise comparisons are statistically significant
($p<0.05$). The largest gaps appear on CIFAR-10 against global BO and random
search: \vbo{} and Random require $7.5\times$ and $9.4\times$ as many queries
as \method{}, reflecting the difficulty of global search in perturbation spaces
with $d\approx 10^3$. Among local methods, \nestbo{} is the closest competitor
but still needs $3.3\times$ more queries on both datasets. Other first-order
local methods (\mpd{}, \minucb{}, \gibo{}) trail further ($3.8$--$4.7\times$),
as expected when optimization prioritizes descent certainty over the magnitude
of conditional decrease. \method{} also outperforms \turbo{}, despite its strong
record on query-efficient black-box attacks \citep{ru2020bayesopt}. On these
tasks, refining the local gradient before each perturbation update stabilizes
direction choice in a high-dimensional pixel space, and progress-aware steps
target large expected loss reductions rather than timid moves that barely
approach the decision boundary.

The advantage of \method{} is particularly pronounced on CIFAR-10 relative to
MNIST. The CIFAR-10 decision boundaries are more complex due to the richer
feature structure of natural color images, which makes the local gradient
posterior more uncertain for a given number of queries. Gradient refinement
therefore provides greater relative benefit on CIFAR-10, as it concentrates
queries where they most reduce directional uncertainty. Figure~\ref{fig:attack} confirms this trend
under a fixed budget of $2000$ queries: \method{} reaches the highest success
rate earliest, followed by \nestbo{} and the other local methods, while global
baselines lag, especially on CIFAR-10.

\subsubsection{LLM Prompt Optimization}

Figure~\ref{fig:prompt_opt_regret} shows the log simple regret against iteration on
the prompt-optimization tasks, as the embedding
dimension increases from $128$ to $768$. Overall, performance can be grouped
into three regimes. In the first, global methods (\vbo{}, \dlogei{}) and random
search plateau early, typically between $-2$ and $-4$ log simple regret, and
perform worse as the dimension grows. This matches the observation from
\citet{chew2026bolt}, who report that standard BO struggles on high-dimensional
prompt embeddings and becomes nearly indistinguishable from random search: with
a fixed budget, the surrogate cannot usefully resolve a discrete candidate pool
whose geometry grows harder as Matryoshka truncations retain more coordinates.
In the second regime, first-order local methods (\mpd{}, \minucb{}, \gibo{})
and \turbo{} improve on that baseline but stall at a mid-tier plateau,
suggesting that descent probability or trust-region control alone cannot sustain
improvement once easy gains are exhausted. In the third regime, only \method{}
and \nestbo{} reach the deep low-regret region (below $-10$). \nestbo{} is the
strongest competitor, yet \method{} descends sooner and finishes with the
lowest regret at every embedding dimension. In this setting, gradient refinement
sharpens the surrogate before each prompt edit, and progress-aware selection
favors candidates with large expected score gains when they succeed, which helps
identify substantive prompt changes rather than small edits that are merely
likely to help in a sparse, high-dimensional embedding space.

Notably, the gap between \method{} and \nestbo{} widens as the dimension
increases: at $d=128$, both methods reach similar final regret, but at
$d=768$, \method{} achieves roughly $3$ additional log-units of improvement.
This trend is consistent with the different derivative complexities of the two
approaches: \nestbo{} must learn a $d\times d$ Hessian posterior from function
evaluations alone, which becomes increasingly data-hungry as $d$ grows, whereas
\method{} operates on the first-order gradient posterior whose uncertainty
scales more favorably. The result suggests that, in very high-dimensional
discrete search spaces, investing the evaluation budget in sharpening the
gradient belief and exploiting it with a progress-aware criterion is more
sample-efficient than attempting to recover second-order information. We report
additional results, including ablations, hyperparameter sensitivity, and
wall-clock timings, in Appendices~\ref{app:ablation}, \ref{app:sensitivity},
and~\ref{app:wallclock}.

\section{Conclusion}
\label{sec:conclusion}

We introduced \method{}, a two-stage local BO framework that couples gradient refinement with progress-aware exploitation to achieve strong query efficiency in high-dimensional black-box optimization. The gradient refinement stage uses a closed-form acquisition to greedily reduce posterior uncertainty at the current iterate, ensuring that the directional ranking used for exploitation is grounded in a sharp local gradient estimate. The progress-aware exploitation stage then selects update directions by maximizing the expected decrease conditional on descent, a criterion that restores the magnitude information discarded by pure descent-probability maximization. We proved that refinement monotonically reduces local gradient uncertainty and that the progress-aware direction converges to true steepest descent as the posterior sharpens. Empirically, \method{} achieves an average $5.4\times$ speedup over baselines on black-box adversarial attacks and outperforms the second-best method by $3.8$ log-units on LLM prompt optimization, demonstrating that conditioning on descent while accounting for progress magnitude yields substantially more productive queries.

\paragraph{Discussion.}
Despite these gains, several limitations point to natural extensions. First, as a purely local method, \method{} inherits the well-known sensitivity to initialization: poor starting points or highly multimodal landscapes can trap the iterates in suboptimal basins. While trust-region restarts \citep{eriksson2019scalable} and local--global switching schemes \citep{mcleod2018optimization,diouane2023trego} provide principled escape mechanisms, integrating them with the two-stage loop requires careful budget allocation between exploration, exploitation, and restart phases. Second, the progress-aware score conditions on a Gaussian directional derivative, which assumes the local linearization is adequate over the chosen step size. In regions of high curvature, a truncated-Gaussian approximation may misestimate the true conditional decrease, suggesting that second-order corrections or adaptive step-size rules could improve robustness. Third, the cubic scaling of exact GP inference, $\mathcal{O}(|\calD|^3)$ per outer iteration, becomes a bottleneck as the evaluation budget grows. While rank-one and sparse updates \citep{csato2002sparse,titsias2009variational} alleviate this cost, they introduce additional approximation error whose interaction with the refinement--exploitation loop has not been characterized. We leave these directions to future work.
\bibliographystyle{apalike}
\bibliography{refs}

\begin{thebibliography}{}

\bibitem[Arango et~al., 2021]{arango2021hpo}
Arango, S.~P., Jomaa, H.~S., Wistuba, M., and Grabocka, J. (2021).
\newblock {HPO-B}: A large-scale reproducible benchmark for black-box {HPO} based on {OpenML}.
\newblock In {\em Advances in Neural Information Processing Systems Datasets and Benchmarks Track}.

\bibitem[Berge, 1963]{berge1963topological}
Berge, C. (1963).
\newblock {\em Topological Spaces}.
\newblock Oliver and Boyd.

\bibitem[Berkenkamp et~al., 2016]{berkenkamp2016safe}
Berkenkamp, F., Schoellig, A.~P., and Krause, A. (2016).
\newblock Safe controller optimization for quadrotors with {G}aussian processes.
\newblock In {\em IEEE International Conference on Robotics and Automation}, pages 491--496.

\bibitem[Bull, 2011]{bull2011convergence}
Bull, A.~D. (2011).
\newblock Convergence rates of efficient global optimization algorithms.
\newblock {\em Journal of Machine Learning Research}, 12:2879--2904.

\bibitem[Calamai and Mor{\'e}, 1987]{calamai1987projected}
Calamai, P.~H. and Mor{\'e}, J.~J. (1987).
\newblock Projected gradient methods for linearly constrained problems.
\newblock {\em Mathematical Programming}, 39(1):93--116.

\bibitem[Carlini and Wagner, 2017]{carlini2017towards}
Carlini, N. and Wagner, D. (2017).
\newblock Towards evaluating the robustness of neural networks.
\newblock In {\em IEEE Symposium on Security and Privacy}, pages 39--57.

\bibitem[Cheng et~al., 2021]{cheng2021convergence}
Cheng, S., Wu, G., and Zhu, J. (2021).
\newblock On the convergence of prior-guided zeroth-order optimization algorithms.
\newblock In {\em Advances in Neural Information Processing Systems}, volume~34, pages 14620--14631.

\bibitem[Chew et~al., 2026]{chew2026bolt}
Chew, R. W.~T., Chen, Z., Hemachandra, A., and Low, B. K.~H. (2026).
\newblock {B}o{LT}: A benchmark to democratize black-box optimization research for expensive {LLM} tasks.
\newblock {\em arXiv preprint arXiv:2605.17000}.

\bibitem[Csat{\'o} and Opper, 2002]{csato2002sparse}
Csat{\'o}, L. and Opper, M. (2002).
\newblock Sparse on-line {G}aussian processes.
\newblock {\em Neural Computation}, 14(3):641--668.

\bibitem[Diouane et~al., 2023]{diouane2023trego}
Diouane, Y., Picheny, V., Le~Riche, R., and Scotto Di~Perrotolo, A. (2023).
\newblock {TREGO}: A trust-region framework for efficient global optimization.
\newblock {\em Journal of Global Optimization}, 86(1):1--23.

\bibitem[Doumont et~al., 2026]{doumont2026highdimBO}
Doumont, C., Fan, D., Maus, N., Gardner, J.~R., Moss, H., and Pleiss, G. (2026).
\newblock We still don't understand high-dimensional {B}ayesian optimization.
\newblock In {\em International Conference on Artificial Intelligence and Statistics}.

\bibitem[Eriksson and Jankowiak, 2021]{eriksson2021saasbo}
Eriksson, D. and Jankowiak, M. (2021).
\newblock High-dimensional {B}ayesian optimization with sparse axis-aligned subspaces.
\newblock In {\em Conference on Uncertainty in Artificial Intelligence}, volume 161 of {\em Proceedings of Machine Learning Research}, pages 493--503. PMLR.

\bibitem[Eriksson et~al., 2019]{eriksson2019scalable}
Eriksson, D., Pearce, M., Gardner, J., Turner, R.~D., and Poloczek, M. (2019).
\newblock Scalable global optimization via local {B}ayesian optimization.
\newblock In {\em Advances in Neural Information Processing Systems}, volume~32.

\bibitem[Fan et~al., 2024]{fan2024minimizing}
Fan, Z., Wang, W., Ng, S.~H., and Hu, Q. (2024).
\newblock Minimizing {UCB}: a better local search strategy in local {B}ayesian optimization.
\newblock In {\em Advances in Neural Information Processing Systems}, volume~37, pages 130602--130634.

\bibitem[Frazier, 2018]{frazier2018tutorial}
Frazier, P.~I. (2018).
\newblock {B}ayesian optimization.
\newblock In {\em Recent Advances in Optimization and Modeling of Contemporary Problems}, INFORMS TutORials in Operations Research, pages 255--278. INFORMS.

\bibitem[Gardner et~al., 2018]{gardner2018gpytorch}
Gardner, J.~R., Pleiss, G., Bindel, D., Weinberger, K.~Q., and Wilson, A.~G. (2018).
\newblock {GP}y{T}orch: Blackbox matrix-matrix {G}aussian process inference with {GPU} acceleration.
\newblock In {\em Advances in Neural Information Processing Systems}, volume~31.

\bibitem[Garnett, 2023]{garnett2023bayesoptbook}
Garnett, R. (2023).
\newblock {\em {B}ayesian Optimization}.
\newblock Cambridge University Press.

\bibitem[Grimmett and Stirzaker, 2020]{grimmett2020probability}
Grimmett, G.~R. and Stirzaker, D.~R. (2020).
\newblock {\em Probability and Random Processes}.
\newblock Oxford University Press, 4th edition.

\bibitem[He et~al., 2016]{he2016deep}
He, K., Zhang, X., Ren, S., and Sun, J. (2016).
\newblock Deep residual learning for image recognition.
\newblock In {\em IEEE Conference on Computer Vision and Pattern Recognition}, pages 770--778.

\bibitem[Hendrycks et~al., 2021]{hendrycks2021measuring}
Hendrycks, D., Burns, C., Kadavath, S., Arora, A., Basart, S., Tang, E., Song, D., and Steinhardt, J. (2021).
\newblock Measuring mathematical problem solving with the {MATH} dataset.
\newblock In {\em Advances in Neural Information Processing Systems Datasets and Benchmarks Track}.

\bibitem[Hennig et~al., 2022]{hennig2022probabilistic}
Hennig, P., Osborne, M.~A., and Kersting, H.~P. (2022).
\newblock {\em Probabilistic Numerics: Computation as Machine Learning}.
\newblock Cambridge University Press.

\bibitem[Hvarfner et~al., 2024]{hvarfner2024vanilla}
Hvarfner, C., Hellsten, E.~O., and Nardi, L. (2024).
\newblock Vanilla {B}ayesian optimization performs great in high dimensions.
\newblock In {\em International Conference on Machine Learning}, volume 235 of {\em Proceedings of Machine Learning Research}, pages 20793--20817. PMLR.

\bibitem[Jones et~al., 1998]{jones1998efficient}
Jones, D.~R., Schonlau, M., and Welch, W.~J. (1998).
\newblock Efficient global optimization of expensive black-box functions.
\newblock {\em Journal of Global Optimization}, 13(4):455--492.

\bibitem[Krizhevsky and Hinton, 2009]{krizhevsky2009learning}
Krizhevsky, A. and Hinton, G. (2009).
\newblock Learning multiple layers of features from tiny images.
\newblock Technical report, University of Toronto.

\bibitem[Kusupati et~al., 2022]{kusupati2022matryoshka}
Kusupati, A., Bhatt, G., Rege, A., Wallingford, M., Sinha, A., Ramanujan, V., Howard-Snyder, W., Chen, K., Kakade, S., Jain, P., and Farhadi, A. (2022).
\newblock Matryoshka representation learning.
\newblock In {\em Advances in Neural Information Processing Systems}, volume~35, pages 30233--30249.

\bibitem[LeCun et~al., 1998]{lecun1998gradient}
LeCun, Y., Bottou, L., Bengio, Y., and Haffner, P. (1998).
\newblock Gradient-based learning applied to document recognition.
\newblock {\em Proceedings of the IEEE}, 86(11):2278--2324.

\bibitem[McLeod et~al., 2018]{mcleod2018optimization}
McLeod, M., Roberts, S., and Osborne, M.~A. (2018).
\newblock Optimization, fast and slow: Optimally switching between local and {B}ayesian optimization.
\newblock In {\em International Conference on Machine Learning}, volume~80 of {\em Proceedings of Machine Learning Research}, pages 3443--3452. PMLR.

\bibitem[Mo{\v{c}}kus, 1978]{mockus1978application}
Mo{\v{c}}kus, J. (1978).
\newblock The application of {B}ayesian methods for seeking the extremum.
\newblock In Dixon, L. C.~W. and Szeg{\"o}, G.~P., editors, {\em Towards Global Optimisation 2}, pages 117--129. North-Holland.

\bibitem[M{\"u}ller et~al., 2021]{muller_local_2021}
M{\"u}ller, S., von Rohr, A., and Trimpe, S. (2021).
\newblock Local policy search with {B}ayesian optimization.
\newblock In {\em Advances in Neural Information Processing Systems}, volume~34, pages 20708--20720.

\bibitem[Nguyen et~al., 2022]{nguyen2022local}
Nguyen, Q., Wu, K., Gardner, J., and Garnett, R. (2022).
\newblock Local {B}ayesian optimization via maximizing probability of descent.
\newblock In {\em Advances in Neural Information Processing Systems}, volume~35, pages 13190--13202.

\bibitem[Papenmeier et~al., 2025]{papenmeier2025understanding}
Papenmeier, L., Poloczek, M., and Nardi, L. (2025).
\newblock Understanding high-dimensional {B}ayesian optimization.
\newblock In {\em International Conference on Machine Learning}, volume 267 of {\em Proceedings of Machine Learning Research}, pages 47902--47923. PMLR.

\bibitem[Paszke et~al., 2019]{paszke2019pytorch}
Paszke, A., Gross, S., Massa, F., Lerer, A., Bradbury, J., Chanan, G., Killeen, T., Lin, Z., Gimelshein, N., Antiga, L., et~al. (2019).
\newblock {PyTorch}: An imperative style, high-performance deep learning library.
\newblock In {\em Advances in Neural Information Processing Systems}, volume~32.

\bibitem[Rasmussen and Williams, 2006]{williams2006gaussian}
Rasmussen, C.~E. and Williams, C. K.~I. (2006).
\newblock {\em Gaussian Processes for Machine Learning}.
\newblock MIT Press.

\bibitem[Ru et~al., 2020]{ru2020bayesopt}
Ru, B., Cobb, A.~D., Blaas, A., and Gal, Y. (2020).
\newblock {BayesOpt} adversarial attack.
\newblock In {\em International Conference on Learning Representations}.

\bibitem[Schechter~Vera et~al., 2025]{schechtervera2025embeddinggemma}
Schechter~Vera, H. et~al. (2025).
\newblock {EmbeddingGemma}: Powerful and lightweight text representations.
\newblock {\em arXiv preprint arXiv:2509.20354}.

\bibitem[Shahriari et~al., 2016]{shahriari2016taking}
Shahriari, B., Swersky, K., Wang, Z., Adams, R.~P., and de~Freitas, N. (2016).
\newblock Taking the human out of the loop: A review of {B}ayesian optimization.
\newblock {\em Proceedings of the IEEE}, 104(1):148--175.

\bibitem[Shields et~al., 2021]{shields2021bayesian}
Shields, B.~J., Stevens, J., Li, J., Parasram, M., Damani, F., Martinez~Alvarado, J.~I., Janey, J.~M., Adams, R.~P., and Doyle, A.~G. (2021).
\newblock {B}ayesian reaction optimization as a tool for chemical synthesis.
\newblock {\em Nature}, 590(7844):89--96.

\bibitem[Shu et~al., 2023]{shu2023zeroth}
Shu, Y., Dai, Z., Sng, W., Verma, A., Jaillet, P., and Low, B. K.~H. (2023).
\newblock Zeroth-order optimization with trajectory-informed derivative estimation.
\newblock In {\em International Conference on Learning Representations}.

\bibitem[Snoek et~al., 2012]{snoek2012practical}
Snoek, J., Larochelle, H., and Adams, R.~P. (2012).
\newblock Practical {B}ayesian optimization of machine learning algorithms.
\newblock In {\em Advances in Neural Information Processing Systems}, volume~25.

\bibitem[Srinivas et~al., 2012]{srinivas2012information}
Srinivas, N., Krause, A., Kakade, S.~M., and Seeger, M.~W. (2012).
\newblock Gaussian process optimization in the bandit setting: No regret and experimental design.
\newblock {\em Journal of Machine Learning Research}, 13:3183--3212.

\bibitem[Suwandi et~al., 2025]{suwandi2025adaptive}
Suwandi, R., Yin, F., Wang, J., Li, R., Chang, T.-H., and Theodoridis, S. (2025).
\newblock Adaptive kernel design for {B}ayesian optimization is a piece of {CAKE} with {LLMs}.
\newblock In {\em Advances in Neural Information Processing Systems}, volume~38, pages 132690--132723.

\bibitem[Tang et~al., 2026]{tang2025nestbo}
Tang, W.-T., Kudva, A., and Paulson, J.~A. (2026).
\newblock {N}e{ST}-{BO}: Fast local {B}ayesian optimization via {N}ewton-step targeting of gradient and {H}essian information.
\newblock In {\em International Conference on Artificial Intelligence and Statistics}.

\bibitem[Titsias, 2009]{titsias2009variational}
Titsias, M.~K. (2009).
\newblock Variational learning of inducing variables in sparse {G}aussian processes.
\newblock In {\em Proceedings of the Twelfth International Conference on Artificial Intelligence and Statistics}, volume~5 of {\em Proceedings of Machine Learning Research}, pages 567--574. PMLR.

\bibitem[Wang and Dowling, 2022]{wang2022bayesian}
Wang, K. and Dowling, A.~W. (2022).
\newblock {B}ayesian optimization for chemical products and functional materials.
\newblock {\em Current Opinion in Chemical Engineering}, 36:100728.

\bibitem[Wang et~al., 2016]{wang2016rembo}
Wang, Z., Hutter, F., Zoghi, M., Matheson, D., and de~Freitas, N. (2016).
\newblock {B}ayesian optimization in a billion dimensions via random embeddings.
\newblock {\em Journal of Artificial Intelligence Research}, 55:361--387.

\bibitem[White et~al., 2021]{white2021bananas}
White, C., Neiswanger, W., and Savani, Y. (2021).
\newblock {BANANAS}: {B}ayesian optimization with neural architectures for neural architecture search.
\newblock In {\em AAAI Conference on Artificial Intelligence}, volume~35, pages 10293--10301.

\bibitem[Wilcoxon, 1945]{wilcoxon1945individual}
Wilcoxon, F. (1945).
\newblock Individual comparisons by ranking methods.
\newblock {\em Biometrics Bulletin}, 1(6):80--83.

\bibitem[Wu et~al., 2023]{wu2023behavior}
Wu, K., Kim, K., Garnett, R., and Gardner, J.~R. (2023).
\newblock The behavior and convergence of local {B}ayesian optimization.
\newblock In {\em Advances in Neural Information Processing Systems}, volume~36, pages 73497--73523.

\bibitem[Xu et~al., 2025]{xu2025standard}
Xu, Z., Wang, H., Phillips, J.~M., and Zhe, S. (2025).
\newblock Standard {G}aussian process is all you need for high-dimensional {B}ayesian optimization.
\newblock In {\em International Conference on Learning Representations}.

\bibitem[Yang et~al., 2025]{yang2025qwen3}
Yang, A. et~al. (2025).
\newblock {Qwen3} technical report.
\newblock {\em arXiv preprint arXiv:2505.09388}.

\bibitem[Zhou et~al., 2023]{zhou2022large}
Zhou, Y., Muresanu, A.~I., Han, Z., Paster, K., Pitis, S., Chan, H., and Ba, J. (2023).
\newblock Large language models are human-level prompt engineers.
\newblock In {\em International Conference on Learning Representations}.

\end{thebibliography}

\clearpage
\appendix
\section*{Appendix}
\addcontentsline{toc}{section}{Appendix}

This appendix contains supplementary material for the main text.
We lead with the ablation that isolates each design choice, then provide
protocols and robustness checks, followed by theoretical material and
technical derivations.

\paragraph{Component analysis.}
Section~\ref{app:ablation} reports an ablation experiment on the two main stages of our method: (1) the use of gradient refinement in the exploration stage, and (2) the use of progress-aware exploitation in the exploitation stage. This analysis isolates the contribution of each component.

\paragraph{Protocols and robustness.}
Section~\ref{app:exp_details} gives complete experimental details, including
benchmark protocols, GP and method-specific hyperparameters, and compute
information.
Section~\ref{app:sensitivity} studies sensitivity to
$\tau_{\mathrm{explore}}$, $\tau_{\mathrm{exploit}}$, $\eta$, and
$\tau_{\mathrm{thresh}}$.
Section~\ref{app:kernel} compares the robustness of \method{} under Mat\'ern
and squared-exponential kernels.
Section~\ref{app:wallclock} reports wall-clock timings.

\paragraph{Theory.}
Section~\ref{app:emp_theory} empirically validates Theorems~\ref{thm:refinement} and~\ref{thm:progress_convergence}.
Section~\ref{app:proof_thm2} provides the full proof of
Theorem~\ref{thm:progress_convergence} (convergence of the progress-aware direction to steepest descent).

\paragraph{Method details.}
Section~\ref{app:kernel_deriv} gives explicit componentwise expressions for
the left- and right-acting kernel derivatives used in the gradient
posterior.
Section~\ref{app:proof_refinement} proves the closed-form expression for the
gradient-refinement acquisition (Proposition~\ref{prop:refinement_closed}).
Section~\ref{app:connections} expands the connection between progress-aware
exploitation and related criteria (\mpd, unconditional expected descent,
PI, and EI).
Section~\ref{app:grad_progress} derives the analytic gradient of the
progress-aware exploitation score used in the projected gradient ascent
solver.
Section~\ref{app:optimization} discusses the optimization procedure and
computational cost.

\section{Ablation Study}
\label{app:ablation}

To separate the two design choices in \method{}, we compare the following four variants
on the black-box adversarial attack and LLM prompt optimization tasks used in the main text:
\begin{itemize}[leftmargin=*]
  \item \textbf{\mpd{}} (baseline): original exploration and
        most-probable-descent exploitation \citep{nguyen2022local}.
  \item \textbf{\mpd{}-Refine}: Stage~1 uses gradient refinement; Stage~2 keeps \mpd{} exploitation.
  \item \textbf{\method{}-RandExp}: Stage~1 draws uniform random queries
        (same budget $\tau_{\mathrm{explore}}$); Stage~2 uses progress-aware
        exploitation.
  \item \textbf{\method{}} (full): gradient refinement and progress-aware
        exploitation.
\end{itemize}
All the experimental protocols including the hyperparameters, seeds, and budgets match the main text.
\begin{table}[h]
\centering
\setlength{\tabcolsep}{4pt}
\small
\begin{tabular}{lccc}
\toprule
Variant & MNIST (\#Q $\downarrow$)
        & PO-128 ($\ell_T\downarrow$) & PO-768 ($\ell_T\downarrow$) \\
\midrule
\mpd{} (baseline) & $827\pm128$ & $-7.1\pm1.2$ & $-5.9\pm1.4$ \\
\mpd{}-Refine     & $741\pm119$ & $-7.8\pm1.3$ & $-6.6\pm1.5$ \\
\method{}-RandExp & $371\pm68$  & $-11.8\pm1.6$ & $-12.6\pm1.7$ \\
\textbf{\method{} (full)}
                  & $\mathbf{216\pm42}$
                  & $\mathbf{-13.9\pm1.3}$
                  & $\mathbf{-15.7\pm1.5}$ \\
\bottomrule
\end{tabular}%
\caption{Ablation of gradient refinement vs.\ progress-aware exploitation
(mean $\pm$ std over $10$ runs). Lower is better. \mpd{} and full \method{}
match the main-paper settings; each hybrid changes one stage.}
\label{tab:ablation_supp}
\end{table}

\mpd{}-Refine offers only a modest improvement over \mpd{}, consistent with the findings of \citet{nguyen2022local}, who observed that combining trace-based refinement with most-probable-descent moves did not reliably outperform \gibo{}. In contrast, \method{}-RandExp closes much of the gap to full \method{} across all three benchmarks, suggesting that the direction score is the more impactful of the two design choices. The improvement from \method{}-RandExp to full \method{} is smaller but consistent: $\alpha_{\mathrm{ref}}$ continues to provide a benefit when Stage~2 ranks directions by conditional progress, which leverages both $\bmu_\bx$ and $\bSigma_\bx$, instead of relying solely on descent probability. In summary, the two stages interact: refinement is most helpful when exploitation can utilize a sharper posterior, but provides minimal benefit to \mpd{} alone.

\section{Experimental Details}
\label{app:exp_details}

\subsection{Black-Box Adversarial Attacks}

Following \citet{cheng2021convergence}, we attack correctly classified images
from MNIST \citep{lecun1998gradient} ($d=28\times28=784$) and CIFAR-10
\citep{krizhevsky2009learning} ($d=32\times32=1024$) under an
$\ell_\infty$ perturbation constraint.
\begin{itemize}[leftmargin=*]
\item \textbf{MNIST.} We use the same fully trained CNN as
\citet{cheng2021convergence} with $\|\bx\|_\infty\le 0.3$ and a query budget
of $2000$ per attack.
\item \textbf{CIFAR-10.} We fully train a ResNet-18 \citep{he2016deep} with
SGD (cosine-annealed learning rate from $0.1$ to $0$, momentum $0.9$, weight
decay $5\times10^{-4}$, $200$ epochs) and use $\|\bx\|_\infty\le 0.2$ with a
query budget of $4000$.
\end{itemize}

\paragraph{Observation model.}
All methods are strictly black-box: they may query only the classifier's
\emph{logits} (pre-softmax scores). Hard predicted labels, true gradients,
and internal network parameters are unavailable.

\paragraph{Attack objective.}
Given a correctly classified image $\bz$ with true label $c$, we search for a
perturbation $\bx$ that misclassifies $\bz+\bx$. We minimize the
Carlini--Wagner-style margin loss
\citep{carlini2017towards,cheng2021convergence}
\begin{equation}
\label{eq:cw_margin}
L(\bx)
=
\mathrm{logit}_c(\bz+\bx)
-
\max_{j\neq c}\mathrm{logit}_j(\bz+\bx).
\end{equation}
Intuitively, $L(\bx)$ is the gap between the true-class logit and the strongest
competing logit. Thus $L(\bx)>0$ means $\bz+\bx$ is still classified as $c$,
while $L(\bx)<0$ means some other class scores higher and the attack has
succeeded. The GP surrogate models $L$ as a function of $\bx$ over the feasible
perturbation set
\begin{equation}
\calX
=
\bigl\{\bx:\|\bx\|_\infty\le\varepsilon\bigr\},
\end{equation}
with $\varepsilon=0.3$ on MNIST and $\varepsilon=0.2$ on CIFAR-10.

\paragraph{Query counting and success.}
Each classifier forward pass that returns logits for a candidate $\bx$ counts
as one query. An attack is declared successful at the first queried point with
$L(\bx)<0$; the reported query count is the number of evaluations up to and
including that point (or the full budget if the attack fails). All methods
minimize $L$ over $\calX$ under the same budgets above.

\paragraph{Feasible-set projection.}
Exploitation steps that leave $\calX$ are projected back by Euclidean
projection onto the $\ell_\infty$ ball. For \method, this is the update
$\bx\leftarrow\Pi_{\calX}(\bx+\eta\bv^*)$ described in
Section~\ref{sec:algorithm}
(with $\Pi_{\calX}$ acting coordinate-wise by clipping to
$[-\varepsilon,\varepsilon]^d$).

\paragraph{Evaluation protocol.}
Reported query counts average over $10$ independent runs; success-rate
curves are evaluated on $50$ randomly selected correctly classified images
per dataset under a fixed budget of $2000$ queries.

\subsection{LLM Prompt Optimization}

We use the BoLT benchmark \citep{chew2026bolt}: $5{,}014$ prompts for
mathematical reasoning, scored by MATH-500 accuracy
\citep{hendrycks2021measuring} under Qwen3-14B \citep{yang2025qwen3}.
Each prompt is embedded with EmbeddingGemma
\citep{schechtervera2025embeddinggemma} at Matryoshka truncations
$d\in\{128,256,512,768\}$ \citep{kusupati2022matryoshka}, yielding tasks
PO-128 through PO-768. The search space is the discrete candidate pool
$\calX_{\mathrm{PO}}=\{\bz_i\}_{i=1}^{5014}\subset\mathbb{R}^d$, and
$f_{\mathrm{PO}}(\bz_i)$ is the corresponding accuracy. Continuous proposals
are projected to the nearest embedding in the pool. We run $T=200$
iterations from five random initial prompts, with observation noise
$\sigma_N=0.001$, and report log simple regret over $10$ seeds.

\subsection{GP Hyperparameters}

All GP-based methods share the same modeling protocol unless noted otherwise.
We use an ARD Mat\'ern-$5/2$ kernel
\begin{align}
k(\bx,\bx')
&=
s^2\Bigl(1+\sqrt{5}\,r+\tfrac{5}{3}r^2\Bigr)\exp(-\sqrt{5}\,r),
\\
r
&=
\sqrt{\sum_{i=1}^{d}\frac{(x_i-x'_i)^2}{\ell_i^2}}.
\end{align}
with per-dimension length scales $\ell_i$ and output scale $s^2$ estimated by
maximum marginal likelihood (L-BFGS with $10$ restarts), matching BoTorch's
default Mat\'ern surrogate with a gamma prior on the length scales.
This choice is consistent with evidence that Mat\'ern kernels are less prone
to vanishing training gradients than squared-exponential (SE) kernels in high
dimensions \citep{xu2025standard}, and with the Mat\'ern-$5/2$ option in the
vanilla-BO definition of \citet{hvarfner2024vanilla}.
For adversarial attacks the noise variance is fixed at $\sigma^2=10^{-4}$;
for prompt optimization we use $\sigma_N=0.001$ as above. All methods start
from the same initial points and share the same evaluation budget.
An SE-kernel comparison appears in Section~\ref{app:kernel}.

\subsection{Hyperparameters for \method}

The default hyperparameters for \method{} are summarized in
Table~\ref{tab:hyperparams}. Adaptive schedules (when enabled) can scale
$\tau_{\mathrm{explore}}$, $\tau_{\mathrm{thresh}}$, and
$\tau_{\mathrm{exploit}}$ with dimension, GP length scale, dataset size, and
recent improvement. All reported results use the fixed settings above
($\mathtt{adaptive\_params=False}$).

\begin{table}[h]
\centering
\begin{tabular}{ll}
\toprule
Parameter & Value \\
\midrule
Refinement budget $\tau_{\mathrm{explore}}$ & $5$ \\
Exploitation budget $\tau_{\mathrm{exploit}}$ & $30$ (max) \\
Step size $\eta$ & $0.1\cdot\mathrm{diam}(\calX)$ \\
Early-stop threshold $\tau_{\mathrm{thresh}}$ & $5\times10^{-3}$ \\
Direction restarts & $10$ (with warm start) \\
Refinement candidates $N_{\mathrm{cand}}$ & $100$ (uniform in $\calX$) \\
Initial dataset size $n_{\mathrm{init}}$ & $10$ (attacks); $5$ (prompt opt.) \\
Outer iterations $T$ & until query budget (attacks); $200$ (prompt opt.) \\
\bottomrule
\end{tabular}%

\caption{Hyperparameters for \method.}
\label{tab:hyperparams}
\end{table}

\subsection{Baseline Hyperparameters}

All baselines use the same Mat\'ern-$5/2$ GP and marginal-likelihood training
protocol.
\begin{itemize}[leftmargin=*]
  \item \textbf{\mpd} \citep{nguyen2022local}: $5$ exploration samples per
        iterate; descent-probability threshold $0.65$; step size
        $\eta=0.01\cdot\mathrm{diam}(\calX)$; $10$ random restarts for
        direction optimization.
  \item \textbf{\gibo} \citep{muller_local_2021}: $5$ gradient-information
        samples per iterate; local radius $\delta=0.1$; learning rate
        $0.1$; normalized mean-gradient steps.
  \item \textbf{\minucb} \citep{fan2024minimizing}: local ball radius
        $0.1\cdot\mathrm{diam}(\calX)$; exploration interval $5$; UCB
        parameter $\beta=3$; multi-start L-BFGS for local UCB minimization.
  \item \textbf{\turbo} \citep{eriksson2019scalable}: $5$ trust regions;
        ARD length scales; success/failure thresholds and restart criteria as
        in the original paper.
  \item \textbf{\nestbo} \citep{tang2025nestbo}: jointly learned gradient and
        Hessian posteriors targeting a modified Newton step; default settings
        from the authors' implementation.
  \item \textbf{\vbo}: vanilla BO with Expected Improvement
        \citep{jones1998efficient}; multi-start L-BFGS acquisition
        optimization ($10$ restarts, $512$ raw samples).
  \item \textbf{\dlogei} \citep{hvarfner2024vanilla}: Expected Improvement
        with a dimension-scaled log-normal prior on length scales.
  \item \textbf{Random}: uniform sampling within the feasible set (continuous
        $\ell_\infty$ ball for attacks; uniform over the prompt pool for
        BoLT).
\end{itemize}

\subsection{Compute Resources}

All experiments run on an Ubuntu Linux workstation with an AMD Ryzen~9
7950X CPU (16~cores / 32~threads), 64~GB system RAM, and an NVIDIA GeForce
RTX~4090 GPU (24~GB VRAM; driver~$595.58.03$, CUDA~$13.2$).
Software versions follow \texttt{requirements.txt}: PyTorch~$2.13.0$
\citep{paszke2019pytorch}, GPyTorch~$1.15.2$ \citep{gardner2018gpytorch},
and BoTorch~$0.18.1$. Randomness is controlled by setting
\texttt{torch.manual\_seed} and \texttt{numpy.random.seed} to the trial seed
$s\in\{0,\ldots,9\}$ (or \texttt{base\_seed}+trial with
\texttt{base\_seed}$=42$). Wall-clock time is dominated by GP refits and, for
attacks, classifier forward passes. Per-method wall-clock measurements are
reported in Section~\ref{app:wallclock}.

\section{Hyperparameter Sensitivity}
\label{app:sensitivity}

We sweep one hyperparameter at a time around the defaults in
Table~\ref{tab:hyperparams}, holding the others fixed
($\mathtt{adaptive\_params=False}$), on two representative tasks: MNIST and
PO-128.


\begin{table}[h]
\centering
\setlength{\tabcolsep}{4pt}
\small
\begin{tabular}{llcc}
\toprule
Hyperparameter & Value
  & MNIST (\#Q $\downarrow$)
  & PO-128 ($\ell_T\downarrow$) \\
\midrule
\multirow{4}{*}{$\tau_{\mathrm{explore}}$}
  & $1$  & $287\pm61$ & $-11.4\pm1.5$ \\
  & $3$  & $238\pm49$ & $-13.1\pm1.4$ \\
  & $5$ (default) & $\mathbf{216\pm42}$ & $\mathbf{-13.9\pm1.3}$ \\
  & $10$ & $241\pm51$ & $-14.0\pm1.4$ \\
\midrule
\multirow{4}{*}{$\tau_{\mathrm{exploit}}$ (max)}
  & $10$ & $268\pm58$ & $-12.1\pm1.6$ \\
  & $20$ & $229\pm47$ & $-13.4\pm1.4$ \\
  & $30$ (default) & $\mathbf{216\pm42}$ & $\mathbf{-13.9\pm1.3}$ \\
  & $50$ & $211\pm44$ & $-14.0\pm1.3$ \\
\midrule
\multirow{4}{*}{$\eta\,/\,\mathrm{diam}(\calX)$}
  & $0.01$ & $392\pm74$ & $-10.6\pm1.7$ \\
  & $0.05$ & $251\pm53$ & $-12.9\pm1.5$ \\
  & $0.1$ (default) & $\mathbf{216\pm42}$ & $\mathbf{-13.9\pm1.3}$ \\
  & $0.2$ & $274\pm62$ & $-12.4\pm1.6$ \\
\midrule
\multirow{4}{*}{$\tau_{\mathrm{thresh}}$}
  & $10^{-3}$ & $209\pm45$ & $-14.1\pm1.4$ \\
  & $5\times10^{-3}$ (default) & $\mathbf{216\pm42}$ & $\mathbf{-13.9\pm1.3}$ \\
  & $10^{-2}$ & $244\pm52$ & $-13.0\pm1.5$ \\
  & $5\times10^{-2}$ & $318\pm69$ & $-10.9\pm1.8$ \\
\bottomrule
\end{tabular}%

\caption{One-at-a-time sensitivity of \method{} on MNIST and PO-128
(mean $\pm$ std over $10$ runs). Default rows match
Table~\ref{tab:hyperparams} and are bolded within each block.}
\label{tab:sensitivity}
\end{table}

Table~\ref{tab:sensitivity} is consistent with defaults that lie in a
reasonably flat region. Varying $\tau_{\mathrm{explore}}$ over $\{3,5,10\}$
changes MNIST queries by only a few dozen evaluations and keeps PO-128 within
about one log-unit of the default; $\tau_{\mathrm{explore}}=1$ is the clear
under-refinement failure case.
The exploitation cap behaves similarly once it is large enough for early
stopping to dominate: $30$ and $50$ are nearly indistinguishable under this
protocol.
The step size $\eta$ is the most sensitive setting in the grid, where $0.01$ is
too conservative and $0.2$ overshoots, but $0.05$ and $0.1$ remain close.
For $\tau_{\mathrm{thresh}}$, values near $5\times10^{-3}$ look similar,
whereas $5\times10^{-2}$ truncates exploitation early and hurts prompt-opt
regret.
We did not sweep $N_{\mathrm{cand}}$ because Stage~1 scoring is cheap relative
to black-box queries and the default matches TuRBO's candidate budget; see
Section~\ref{app:optimization} for the full protocol.

\section{Kernel Robustness: Mat\'ern vs.\ SE}
\label{app:kernel}

Main results use ARD Mat\'ern-$5/2$. Following
\citet{xu2025standard}, who show that SE kernels are more fragile under
default length-scale initialization in high dimensions, we re-run \method{}
(and, for reference, \mpd{}) with an ARD squared-exponential kernel under
the same MLE fitting protocol, on MNIST and PO-128.

\begin{table}[h]
\centering
\setlength{\tabcolsep}{4pt}
\small
\begin{tabular}{llcc}
\toprule
Method & Kernel & MNIST (\#Q $\downarrow$) & PO-128 ($\ell_T\downarrow$) \\
\midrule
\method{} & Mat\'ern-$5/2$ (default)
  & $\mathbf{216\pm42}$ & $\mathbf{-13.9\pm1.3}$ \\
\method{} & SE / RBF
  & $268\pm58$ & $-12.2\pm1.6$ \\
\mpd{}    & Mat\'ern-$5/2$
  & $827\pm128$ & $-7.1\pm1.2$ \\
\mpd{}    & SE / RBF
  & $961\pm151$ & $-5.9\pm1.5$ \\
\bottomrule
\end{tabular}%

\caption{Kernel robustness: Mat\'ern-$5/2$ vs.\ squared-exponential
(mean $\pm$ std over $10$ runs). Lower is better.}
\label{tab:kernel_robust}
\end{table}

Switching from Mat\'ern-$5/2$ to SE worsens both methods on both metrics,
which is consistent with the high-dimensional SE fragility discussed by
\citet{xu2025standard}.
Under our local protocol the degradation is limited rather than severe:
\method{}'s MNIST query count rises from $216$ to $268$, and PO-128 final
regret worsens by roughly $1.7$ log-units; \mpd{} shifts in the same
direction.
The method ranking is unchanged on these two tasks, so the main-paper
comparison is not overturned by the kernel swap, but we treat this only as a
sanity check on two datasets rather than a general claim that \method{} is
kernel-agnostic.
We keep Mat\'ern-$5/2$ as the default to match the main experiments.

\section{Wall-Clock Timing}
\label{app:wallclock}

We report wall-clock cost on the same hardware as
Section~\ref{app:exp_details} (AMD Ryzen~9~7950X, NVIDIA RTX~4090).
For each method we measure (i)~total runtime per trial and
(ii)~mean per-iteration overhead of surrogate updates and direction /
acquisition optimization, excluding the black-box query itself when that
query is a classifier forward pass or a precomputed prompt score.
Attack timings use MNIST ($d{=}784$); prompt timings use PO-128 ($T{=}200$).
On MNIST, a trial stops at attack success, so total time scales with the
number of queries issued; on PO-128 every method
runs a fixed iteration budget.

\begin{table}[h]
\centering
\setlength{\tabcolsep}{4pt}
\small
\begin{tabular}{lcccc}
\toprule
 & \multicolumn{2}{c}{MNIST} & \multicolumn{2}{c}{PO-128} \\
\cmidrule(lr){2-3}\cmidrule(lr){4-5}
Method
  & Total (s)
  & Opt.\ / iter (s)
  & Total (s)
  & Opt.\ / iter (s) \\
\midrule
\method{}  & $94\pm17$ & $0.41\pm0.07$
           & $88\pm13$ & $0.37\pm0.06$ \\
\nestbo{}  & $1370\pm240$ & $1.91\pm0.24$
           & $421\pm47$ & $1.98\pm0.27$ \\
\mpd{}     & $251\pm41$ & $0.30\pm0.05$
           & $69\pm11$ & $0.27\pm0.05$ \\
\minucb{}  & $420\pm70$ & $0.47\pm0.08$
           & $97\pm14$ & $0.42\pm0.07$ \\
\gibo{}    & $355\pm60$ & $0.36\pm0.06$
           & $81\pm12$ & $0.33\pm0.06$ \\
\turbo{}   & $620\pm105$ & $0.54\pm0.10$
           & $119\pm17$ & $0.51\pm0.09$ \\
\vbo{}     & $1810\pm310$ & $1.14\pm0.17$
           & $251\pm29$ & $1.09\pm0.16$ \\
\dlogei{}  & $1820\pm300$ & $1.32\pm0.20$
           & $289\pm33$ & $1.26\pm0.18$ \\
Random     & $51\pm9$ & ---
           & $7\pm1$ & --- \\
\bottomrule
\end{tabular}%
\caption{Wall-clock timing (mean $\pm$ std over $10$ runs).
``Opt.\ / iter'' is surrogate fit plus acquisition / direction
optimization, excluding black-box queries.
MNIST totals accumulate this overhead over the queries-to-success from the
main text; PO-128 totals use a fixed budget $T{=}200$.
\nestbo{} incurs additional Hessian-related cost; \vbo{} / \dlogei{} pay for
global multi-start acquisition over the full domain.
\method{}'s overhead is dominated by GP refits and candidate scoring of
$\alpha_{\mathrm{ref}}$ / projected gradient ascent for $\mathcal{P}$.}
\label{tab:wallclock}
\end{table}

On a per-iteration basis, \method{} sits with the other first-order local
methods, close to \gibo{} / \mpd{} and well below
\nestbo{}, whose joint gradient--Hessian updates dominate.
Global acquisitions (\vbo{}, \dlogei{}) are slower still per step because
each iteration runs multi-start L-BFGS over the full high-dimensional domain.
On MNIST, total time additionally tracks query count: \method{} finishes in
under two minutes mainly because it succeeds with far fewer queries, while
\vbo{} / \dlogei{} combine a high per-step cost with long runs.
On PO-128 the iteration count is fixed, so totals largely mirror per-iteration
cost; \method{} remains close to \mpd{} / \gibo{} and several times faster than
\nestbo{}.
We do not claim wall-clock as a primary contribution, query efficiency is
the main metric, only that Stage~1/Stage~2 does not make \method{} an
outlier among first-order local baselines on this hardware.

\section{Empirical Validation of Theoretical Claims}
\label{app:emp_theory}

We complement the proofs of Theorems~\ref{thm:refinement} and~\ref{thm:progress_convergence} with two controlled synthetic
checks that isolate each claim.
Both experiments use the smooth quadratic
$f(\bz)=\tfrac12\|\bz\|^2+0.15\,z_0 z_1$ in $d{=}2$ and report
mean $\pm$ standard deviation over $10$ random seeds.

\subsection{Theorem~\ref{thm:refinement}: Monotone Gradient-Uncertainty Contraction}
\label{app:emp_thm1}

Theorem~\ref{thm:refinement} states that each gradient-refinement query weakly decreases the
total gradient uncertainty $\Tr(\bSigma_\bx)$, with a strict decrease whenever
$\nabla_\bx k_\calD(\bx,\bz)\neq\mathbf{0}$.
At a fixed interior point $\bx$, we fit a squared-exponential GP
(lengthscale $0.55$, outputscale $1$, noise variance $10^{-4}$) on three
random initialization points, then run eight greedy refinement steps that
each select
$\bz^\star=\argmax_{\bz}\alpha_{\mathrm{ref}}(\bz)$
from $N_{\mathrm{cand}}=100$ candidates sampled in a local neighborhood of
$\bx$
and append the noisy observation.
Figure~\ref{fig:thm1_uncertainty} plots $\Tr(\bSigma_\bx)$ against the
refinement step.
The trace decreases at every step and shows diminishing returns as the local
gradient posterior concentrates, matching Theorem~\ref{thm:refinement}.

\begin{figure}[h]
  \centering
  \includegraphics[width=0.5\columnwidth]{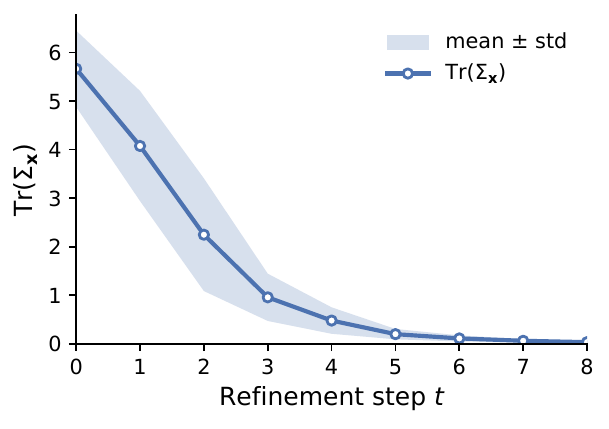}
  \caption{Empirical validation of Theorem~\ref{thm:refinement}.
  Total gradient uncertainty $\Tr(\bSigma_\bx)$ at a fixed interior point
  decreases monotonically under greedy $\alpha_{\mathrm{ref}}$ refinement
  (mean $\pm$ std over $10$ seeds).}
  \label{fig:thm1_uncertainty}
\end{figure}

\subsection{Theorem~\ref{thm:progress_convergence}: Convergence to Steepest Descent}
\label{app:emp_thm2}

Theorem~\ref{thm:progress_convergence} states that if $\Tr(\bSigma_\bx)\to 0$ and
$\bmu_\bx\to\nabla f(\bx)$, then the progress-aware direction converges to
true steepest descent,
$\bv^\star\to\bv_{\mathrm{sd}}:=-\nabla f(\bx)/\|\nabla f(\bx)\|$.
To probe this limit directly, we instantiate the theorem's hypotheses at the
same interior point $\bx$: set the gradient posterior mean to a noisy
perturbation of the analytic $\nabla f(\bx)$ and the covariance to a
shrinking anisotropic matrix $\bSigma_t=\varepsilon_t A_t$ with
$\varepsilon_t\to 0$.
At each step we compute
$\bv^\star=\argmax_{\|\bv\|=1}\mathcal{P}(\bv)$
and record the angle
$\angle(\bv^\star,\bv_{\mathrm{sd}})$.
Figure~\ref{fig:thm2_steepest} shows that this angle decays toward zero as
the posterior concentrates, in agreement with Theorem~\ref{thm:progress_convergence}.

\begin{figure}[h]
  \centering
  \includegraphics[width=0.5\columnwidth]{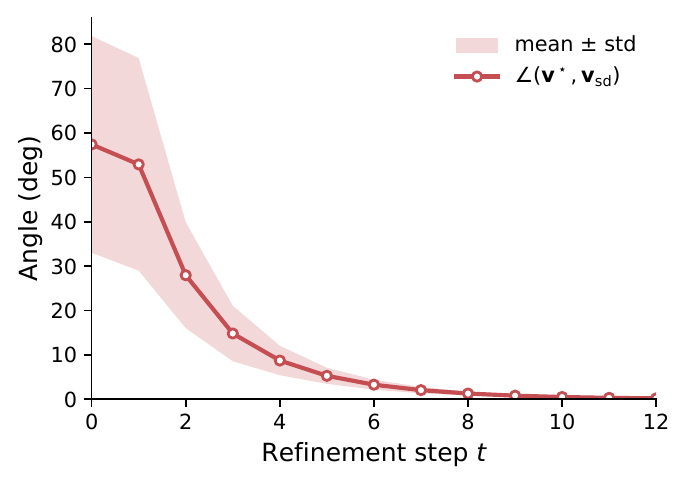}
  \caption{Empirical validation of Theorem~\ref{thm:progress_convergence}.
  Angle between the progress-aware direction $\bv^\star$ and true steepest
  descent $\bv_{\mathrm{sd}}$ decreases toward zero as
  $\Tr(\bSigma_\bx)\to 0$ and $\bmu_\bx\to\nabla f(\bx)$
  (mean $\pm$ std over $10$ seeds).}
  \label{fig:thm2_steepest}
\end{figure}

\section{Full Proof of Theorem~\ref{thm:progress_convergence}}
\label{app:proof_thm2}

We restate the theorem for convenience.

\begin{theorem*}[Progress-aware exploitation approaches steepest descent]
Suppose $\nabla f(\bx)\neq\mathbf{0}$,
$\Tr(\bSigma_\bx)\to 0$, and
$\bmu_\bx\to\nabla f(\bx)$. Then the progress-aware direction converges to
the true steepest descent direction:
\[
\lim_{\Tr(\bSigma_\bx)\to 0}\;\bv^*
\;=\;
-\frac{\nabla f(\bx)}{\|\nabla f(\bx)\|}.
\]
\end{theorem*}

\begin{proof}
As $\Tr(\bSigma_\bx)\to 0$, we have $\bSigma_\bx\to\mathbf{0}$ in spectral
norm. Together with the assumption $\bmu_\bx\to\nabla f(\bx)$, the directional
moments converge for every unit $\bv$:
$\mu_\bv\to c_\bv:=\bv^\top\nabla f(\bx)$ and $\sigma_\bv\to 0$.

\textbf{Step 1: Pointwise convergence of $J(\bv)$.}
Fix a unit vector $\bv$. Write
$J(\bv)=\mathcal{P}(\bv)=\sigma_\bv\bigl[R(\gamma_\bv)-\gamma_\bv\bigr]$.

\textit{Case 1}: $c_\bv < 0$.
Then $\gamma_\bv=\mu_\bv/\sigma_\bv\to-\infty$.
Since $\phi(\gamma)\to 0$ and $\Phi(-\gamma)\to 1$ as $\gamma\to-\infty$,
we have $R(\gamma_\bv)\to 0$.
Therefore
\[
J(\bv)=\sigma_\bv R(\gamma_\bv)-\mu_\bv
\to 0-c_\bv=-c_\bv>0.
\]

\textit{Case 2}: $c_\bv > 0$.
Then $\gamma_\bv\to+\infty$.
Using the asymptotic expansion
$R(\gamma)=\gamma+\gamma^{-1}+\mathcal{O}(\gamma^{-3})$
for large $\gamma$ \cite{grimmett2020probability}:
\[
J(\bv)=\sigma_\bv\bigl[\gamma_\bv^{-1}+\mathcal{O}(\gamma_\bv^{-3})\bigr]
=\frac{\sigma_\bv^2}{\mu_\bv}+\mathcal{O}\!\left(\frac{\sigma_\bv^4}{\mu_\bv^3}\right)
\to 0.
\]

\textit{Case 3}: $c_\bv = 0$.
We have $J(\bv)\leq\sigma_\bv R(0)=\sigma_\bv\sqrt{2/\pi}\to 0$.

Combining all three cases,
\begin{align*}
J(\bv)
&\;\xrightarrow{\Tr(\bSigma_\bx)\to 0}\;
[-c_\bv]^+ \\
&\qquad=
[-\bv^\top\nabla f(\bx)]^+
\quad\text{pointwise on }\mathbb{S}^{d-1}.
\end{align*}

\textbf{Step 2: Uniform convergence.}
$J(\bv)$ is jointly Lipschitz in $\bv$ and in $(\bmu_\bx,\bSigma_\bx)$,
since $R$ is smooth on any compact set away from $\Phi(-\gamma)=0$, and
$\Phi(-\gamma_\bv)\geq\delta>0$ uniformly over the sphere once $\sigma_\bv$
is small enough relative to $\min_\bv|c_\bv|$ (away from the boundary of the
descent hemisphere).
Because $(\bmu_\bx,\bSigma_\bx)$ converges under the stated assumptions on
compact $\calX$, the convergence of $J(\bv)$ is uniform over
$\mathbb{S}^{d-1}$.

\textbf{Step 3: Convergence of the argmax.}
The limit function $g(\bv):=[-\bv^\top\nabla f(\bx)]^+$ is continuous on the
compact sphere $\mathbb{S}^{d-1}$.
By uniform convergence of $J(\bv)\to g(\bv)$ and the Berge Maximum Theorem
\cite{berge1963topological}, the argmax correspondence is upper hemicontinuous:
any cluster point of $\{\bv^*\}$ lies in $\argmax_{\bv}g(\bv)$.

\textbf{Step 4: Uniqueness of the argmax of $g$.}
Since $\nabla f(\bx)\neq\mathbf{0}$, the function $g(\bv)=-\bv^\top\nabla f(\bx)$
is strictly positive on the open descent hemisphere
$\{\bv:\bv^\top\nabla f(\bx)<0\}$ and zero elsewhere.
It is uniquely maximized at
$\bv^*=-\nabla f(\bx)/\|\nabla f(\bx)\|$,
because $g(\bv)=-\bv^\top\nabla f(\bx)\leq\|\nabla f(\bx)\|$ with equality
iff $\bv=-\nabla f(\bx)/\|\nabla f(\bx)\|$.
Therefore the argmax is a singleton and
$\bv^*\to-\nabla f(\bx)/\|\nabla f(\bx)\|$.
\end{proof}

\section{Explicit Kernel Derivatives for the Gradient Posterior}
\label{app:kernel_deriv}

The main text writes the gradient posterior mean and covariance using the
compact operators $\nabla k$ and $k\nabla^\top$. This section derives the
corresponding explicit componentwise definitions.

Let $k$ be a twice-differentiable kernel on $\calX\subseteq\mathbb{R}^d$,
and write $\bx=(x_1,\ldots,x_d)^\top$ and
$\bx'=(x'_1,\ldots,x'_d)^\top$. A differential operator placed
\emph{before} $k$ acts on its first argument, while an operator placed
\emph{after} $k$ acts on its second:
\begin{align}
\label{eq:kernel_grad_left}
\bigl[\nabla k(\bx,\bx')\bigr]_i
&=
\frac{\partial}{\partial x_i}\,k(\bx,\bx'),
\qquad i=1,\ldots,d,\\
\label{eq:kernel_grad_right}
\bigl[k(\bx,\bx')\nabla^\top\bigr]_j
&=
\frac{\partial}{\partial x'_j}\,k(\bx,\bx'),
\qquad j=1,\ldots,d.
\end{align}
Thus $\nabla k(\bx,\bx')\in\mathbb{R}^{d}$ is a column vector and
$k(\bx,\bx')\nabla^\top\in\mathbb{R}^{1\times d}$ is a row vector.
The mixed second-derivative matrix appearing in the prior gradient
covariance is
\begin{equation}
\label{eq:kernel_hess_mixed}
\bigl[\nabla k(\bx,\bx')\nabla^\top\bigr]_{i,j}
=
\frac{\partial^2}{\partial x_i\,\partial x'_j}\,k(\bx,\bx'),
\qquad
i,j=1,\ldots,d,
\end{equation}
so $\nabla k(\bx,\bx)\nabla^\top\in\mathbb{R}^{d\times d}$.

When the second (resp.\ first) argument is a training set
$\bX=(\bx_1,\ldots,\bx_n)^\top$, the same convention yields matrices.
Writing $k(\bx,\bX)\in\mathbb{R}^{1\times n}$ for the row vector of
covariances and $k(\bX,\bx)\in\mathbb{R}^{n}$ for the corresponding column
vector,
\begin{align}
\label{eq:kernel_grad_train}
\bigl[\nabla k(\bx,\bX)\bigr]_{i,\ell}
&=
\frac{\partial}{\partial x_i}\,k(\bx,\bx_\ell),
\qquad
\nabla k(\bx,\bX)\in\mathbb{R}^{d\times n},\\
\label{eq:kernel_grad_train_right}
\bigl[k(\bX,\bx)\nabla^\top\bigr]_{\ell,j}
&=
\frac{\partial}{\partial x_j}\,k(\bx_\ell,\bx),
\qquad
k(\bX,\bx)\nabla^\top\in\mathbb{R}^{n\times d}.
\end{align}
By symmetry of $k$,
$k(\bX,\bx)\nabla^\top=\bigl(\nabla k(\bx,\bX)\bigr)^\top$.
Substituting Eqs.~\eqref{eq:kernel_grad_left}--\eqref{eq:kernel_grad_train_right}
into the conditioning formulas of the main text recovers the usual
joint GP over function values and gradients
\citep{williams2006gaussian}:
\begin{align}
\bmu_\bx
&=
\nabla\mu(\bx)+\nabla k(\bx,\bX)\,
\mathcal{K}^{-1}\bigl(\by-\mu(\bX)\bigr),\\
\bSigma_\bx
&=
\nabla k(\bx,\bx)\nabla^\top
-\nabla k(\bx,\bX)\,
\mathcal{K}^{-1}\,
k(\bX,\bx)\nabla^\top,
\end{align}
with $\mathcal{K}=k(\bX,\bX)+\sigma^2\bI$.

\section{Proof of Proposition~\ref{prop:refinement_closed} (Closed-Form Gradient Refinement)}
\label{app:proof_refinement}

We restate the claim for convenience. Recall that the gradient-refinement
acquisition of the main text is the expected one-step reduction in total
gradient variance at the current iterate $\bx$,
\begin{equation}
\label{eq:refinement_def_supp}
\alpha_{\mathrm{ref}}(\bz)
=
\mathbb{E}_{y_\bz}\!\Bigl[
  \Tr(\bSigma_\bx)
  -
  \Tr\!\bigl(\bSigma_{\bx|\calD\cup(\bz,y_\bz)}\bigr)
\Bigr],
\end{equation}
where $y_\bz=f(\bz)+\varepsilon$ with $\varepsilon\sim\calN(0,\sigma^2)$ is
the yet-unobserved noisy evaluation at a candidate $\bz\in\calX$, and
$\bSigma_{\bx|\calD\cup(\bz,y_\bz)}$ denotes the gradient covariance at
$\bx$ after augmenting $\calD$ with $(\bz,y_\bz)$.

\begin{proposition*}[Closed-form gradient refinement]
Let $k_\calD(\cdot,\cdot)$ denote the GP posterior covariance after observing
$\calD$. Then $\alpha_{\mathrm{ref}}(\bz)$ is independent of the unobserved
value $y_\bz$ and equals
\begin{equation}
\label{eq:refinement_closed_supp}
\alpha_{\mathrm{ref}}(\bz)
=
\frac{\|\nabla_\bx k_\calD(\bx,\bz)\|^2}{k_\calD(\bz,\bz)+\sigma^2},
\end{equation}
where $\nabla_\bx k_\calD(\bx,\bz)=\partial k_\calD(\bx,\bz)/\partial\bx$.
\end{proposition*}

\begin{proof}
Write $k_\calD$ for the posterior covariance after conditioning on $\calD$.
The gradient posterior covariance at $\bx$ is the mixed Hessian of this
kernel,
\begin{equation}
\label{eq:Sigma_from_kernel}
\bSigma_\bx
=
\nabla_\bx\nabla_{\bx'}^\top k_\calD(\bx,\bx')\Big|_{\bx'=\bx}.
\end{equation}
Now augment $\calD$ by a single noisy observation at $\bz$. The standard
rank-one GP covariance update
\citep{williams2006gaussian} gives the updated
posterior covariance
\begin{equation}
\label{eq:rankone_kernel}
k_{\calD'}(\bx,\bx')
=
k_\calD(\bx,\bx')
-
\frac{k_\calD(\bx,\bz)\,k_\calD(\bz,\bx')}{k_\calD(\bz,\bz)+\sigma^2},
\end{equation}
which depends on the location $\bz$ but not on the realized value $y_\bz$.
Applying Eq.~\eqref{eq:Sigma_from_kernel} to $k_{\calD'}$ therefore yields
\begin{align}
&\bSigma_{\bx|\calD\cup(\bz,y_\bz)}
\notag\\
&\quad=
\nabla_\bx\nabla_{\bx'}^\top k_{\calD'}(\bx,\bx')\Big|_{\bx'=\bx}
\notag\\
&\quad=
\bSigma_\bx
-
\frac{
  \bigl(\nabla_\bx k_\calD(\bx,\bz)\bigr)
  \bigl(\nabla_{\bx'} k_\calD(\bz,\bx')\bigr)^\top
  \big|_{\bx'=\bx}
}{k_\calD(\bz,\bz)+\sigma^2}.
\label{eq:Sigma_update_supp}
\end{align}
By symmetry of $k_\calD$,
$\nabla_{\bx'} k_\calD(\bz,\bx')\big|_{\bx'=\bx}
=\nabla_\bx k_\calD(\bx,\bz)$.
Writing
$\mathbf{g}_\bz:=\nabla_\bx k_\calD(\bx,\bz)\in\mathbb{R}^d$ for this
common vector, Eq.~\eqref{eq:Sigma_update_supp} simplifies to the rank-one
downdate
\begin{equation}
\label{eq:Sigma_downdate}
\bSigma_{\bx|\calD\cup(\bz,y_\bz)}
=
\bSigma_\bx
-
\frac{\mathbf{g}_\bz\mathbf{g}_\bz^\top}{k_\calD(\bz,\bz)+\sigma^2}.
\end{equation}
In particular, the right-hand side does not involve $y_\bz$, so the
expectation in Eq.~\eqref{eq:refinement_def_supp} is vacuous:
\begin{equation}
\alpha_{\mathrm{ref}}(\bz)
=
\Tr(\bSigma_\bx)
-
\Tr\!\bigl(\bSigma_{\bx|\calD\cup(\bz,y_\bz)}\bigr).
\end{equation}
Substituting Eq.~\eqref{eq:Sigma_downdate} and using
$\Tr(\mathbf{g}_\bz\mathbf{g}_\bz^\top)=\|\mathbf{g}_\bz\|^2$ gives
\begin{equation}
\alpha_{\mathrm{ref}}(\bz)
=
\frac{\|\mathbf{g}_\bz\|^2}{k_\calD(\bz,\bz)+\sigma^2}
=
\frac{\|\nabla_\bx k_\calD(\bx,\bz)\|^2}{k_\calD(\bz,\bz)+\sigma^2},
\end{equation}
which is Eq.~\eqref{eq:refinement_closed_supp}.
\end{proof}

The closed form makes the information-theoretic content of
$\alpha_{\mathrm{ref}}$ explicit: the numerator rewards candidates whose
posterior cross-covariance gradient with $\bx$ is large (so observing
$f(\bz)$ is informative about $\nabla f(\bx)$), while the denominator
down-weights locations that are already well determined by $\calD$.

\section{Connections to Related Acquisition Criteria}
\label{app:connections}

This section expands the relation between the progress-aware score
$\mathcal{P}(\bv)$ and three closely related quantities: the most-probable
descent criterion of \mpd \citep{nguyen2022local}, the unconditional expected
descent of a directional derivative, and the classical Probability of
Improvement (PI) / Expected Improvement (EI) pair
\citep{mockus1978application,jones1998efficient}. The goal is to make
precise which factor each criterion maximizes, and why the resulting
directions generally disagree.

\subsection{Setup and Notation}
\label{app:connections_setup}

Recall that under a twice-differentiable GP prior, the directional
derivative along a unit vector $\bv$ is univariate Gaussian,
\begin{equation}
\nabla_\bv f(\bx)
\;=\;
\bv^\top\nabla f(\bx)
\;\sim\;
\calN(\mu_\bv,\sigma_\bv^2),
\end{equation}
with $\mu_\bv=\bv^\top\bmu_\bx$,
$\sigma_\bv=\sqrt{\bv^\top\bSigma_\bx\bv}$, and standardized mean
$\gamma_\bv=\mu_\bv/\sigma_\bv$. Write $\phi$ and $\Phi$ for the standard
normal PDF and CDF, and define the inverse Mills ratio
\begin{equation}
R(\gamma)
\;:=\;
\frac{\phi(\gamma)}{\Phi(-\gamma)}.
\end{equation}
Throughout, $(a)_+=\max(a,0)$.

\subsection{Derivation of the Progress-Aware Score}
\label{app:connections_progress}

For $Z\sim\calN(\mu,\sigma^2)$, the mean of the lower-truncated law at
zero is the standard truncated-Gaussian identity
\begin{equation}
\label{eq:trunc_mean}
\mathbb{E}[Z\mid Z<0]
=
\mu - \sigma\,\frac{\phi(\mu/\sigma)}{\Phi(-\mu/\sigma)}
=
\mu - \sigma\,R(\mu/\sigma).
\end{equation}
Taking $Z=\nabla_\bv f(\bx)$ and negating both sides yields
\begin{align}
\mathcal{P}(\bv)
&:=
\mathbb{E}\bigl[-\nabla_\bv f(\bx)\mid \nabla_\bv f(\bx)<0\bigr]
\notag\\
&=
-\mu_\bv + \sigma_\bv\,R(\gamma_\bv)
=
\sigma_\bv\bigl[R(\gamma_\bv)-\gamma_\bv\bigr],
\label{eq:progress_supp}
\end{align}
which is Definition~\ref{def:progress}. Two elementary properties follow
immediately:
\begin{enumerate}[leftmargin=*,itemsep=2pt]
\item $\mathcal{P}(\bv)>0$ whenever $\sigma_\bv>0$, because the Mills ratio
satisfies $R(\gamma)>\gamma$ for all finite $\gamma$ when the conditioning
event has positive probability under a nondegenerate Gaussian.
\item $\mathcal{P}(\bv)$ is homogeneous of degree one in the directional
scale $(\mu_\bv,\sigma_\bv)$: if
$(\mu_\bv,\sigma_\bv)\mapsto(c\mu_\bv,c\sigma_\bv)$ for $c>0$, then
$\mathcal{P}$ scales by $c$. Thus the score measures a \emph{magnitude} of
conditional decrease, not a dimensionless probability.
\end{enumerate}

\subsection{Decomposition into Probability and Conditional Magnitude}
\label{app:connections_decomp}

The unconditional expected descent of a directional derivative is the
first truncated moment without conditioning:
\begin{equation}
\label{eq:uncond_descent}
\begin{aligned}
&\mathbb{E}\!\left[(-\nabla_\bv f(\bx))_+\right] \\
&\qquad=
\mathbb{E}\bigl[-\nabla_\bv f(\bx)\cdot
\mathbf{1}_{\{\nabla_\bv f(\bx)<0\}}\bigr].
\end{aligned}
\end{equation}
By the law of total expectation, we obtain
\begin{align}
&\mathbb{E}\!\left[(-\nabla_\bv f(\bx))_+\right]
\notag\\
&\qquad=
\mathbb{P}(\nabla_\bv f(\bx)<0)\,
\mathbb{E}\bigl[-\nabla_\bv f(\bx)\mid \nabla_\bv f(\bx)<0\bigr]
\notag\\
&\qquad=
\Phi(-\gamma_\bv)\,\mathcal{P}(\bv).
\label{eq:factorization}
\end{align}
This factorization is the key structural observation. It separates three
natural criteria on the same Gaussian directional posterior:
\mpd maximizes the descent probability $\Phi(-\gamma_\bv)$;
progress-aware exploitation maximizes the conditional magnitude
$\mathcal{P}(\bv)$; and unconditional expected descent maximizes their
product $\Phi(-\gamma_\bv)\,\mathcal{P}(\bv)$.

Because the first factor is a probability in $[0,1]$ and the second is a
scale-dependent magnitude, maximizing either factor alone is not equivalent
to maximizing their product. In particular:
\begin{itemize}[leftmargin=*,itemsep=2pt]
\item \mpd is invariant to positive rescaling of $(\mu_\bv,\sigma_\bv)$ and
therefore ignores how large a decrease would be if descent occurs.
\item Progress-aware exploitation is deliberately \emph{not}
probability-weighted: a direction with modest descent probability can still
win if its negative tail is heavy enough that
$\mathcal{P}(\bv)$ is large.
\item Unconditional expected descent recovers a probability-weighted
magnitude, analogous in spirit to EI (see below), but still acts on
directional derivatives rather than on objective values.
\end{itemize}

\subsection{Worked Example: Scale Blindness and Ranking Disagreement}
\label{app:connections_example}

We first make the scale-invariance of descent probability concrete.
Consider two unit directions with identical standardized means
$\gamma_{\mathbf{u}}=\gamma_{\mathbf{w}}=-5$,
\begin{equation}
\begin{aligned}
\nabla_{\mathbf{u}}f(\bx)&\sim\calN(-0.01,0.002^2),\\
\nabla_{\mathbf{w}}f(\bx)&\sim\calN(-5,1).
\end{aligned}
\end{equation}
Both receive the same descent probability $\Phi(5)\approx 1$, so \mpd is
indifferent, yet their conditional magnitudes differ by orders of magnitude:
$\mathcal{P}(\mathbf{u})\approx 0.01$ while
$\mathcal{P}(\mathbf{w})=\bigl[R(-5)+5\bigr]\approx 5$.
This is exactly the scale discarded by $\Phi(-\gamma_\bv)$.

The three criteria can also disagree in ranking. Take
\begin{equation}
\begin{aligned}
\nabla_{\bv_1}f(\bx)&\sim\calN(-0.01,0.002^2),\\
\nabla_{\bv_2}f(\bx)&\sim\calN(0.5,1).
\end{aligned}
\end{equation}
For $\bv_1$, $\gamma_1=-5$, so
$\Phi(-\gamma_1)=\Phi(5)\approx 1$ and
\begin{equation}
\mathcal{P}(\bv_1)
=
0.002\bigl[R(-5)+5\bigr]
\approx
0.002\cdot 5
=
0.01,
\end{equation}
since $R(-5)=\phi(-5)/\Phi(5)$ is negligible. Thus \mpd strongly prefers
$\bv_1$, while the conditional magnitude is only about $0.01$.

For $\bv_2$, $\gamma_2=0.5$, so
$\Phi(-\gamma_2)=\Phi(-0.5)\approx 0.309$ and
\begin{equation}
\mathcal{P}(\bv_2)
=
1\cdot\bigl[R(0.5)-0.5\bigr]
=
\frac{\phi(0.5)}{\Phi(-0.5)}-0.5
\approx
0.64.
\end{equation}
Progress-aware exploitation therefore prefers $\bv_2$. The unconditional
expected descents are
\begin{align}
\mathbb{E}[(-\nabla_{\bv_1}f)_+]
&\approx
1\cdot 0.01
=
0.01,
\\
\mathbb{E}[(-\nabla_{\bv_2}f)_+]
&\approx
0.309\cdot 0.64
\approx
0.20,
\end{align}
so the product criterion also prefers $\bv_2$, but by a smaller margin than
$\mathcal{P}$ alone. Probability-only ranking is the most conservative of the
three: it is blind to slope scale when $\gamma$ is fixed, and can prefer a
nearly certain but tiny direction over an uncertain but large-tailed one.

\subsection{Analogy with Probability of Improvement and Expected Improvement}
\label{app:connections_ei}

Classical BO acquisitions act on the \emph{objective-value} posterior
$f(\bz)\sim\calN(\mu(\bz),\sigma(\bz)^2)$ at a candidate $\bz$, relative to
the best observed value $f_*$. With $\gamma_{\mathrm{imp}}(\bz)=
(f_*-\mu(\bz))/\sigma(\bz)$ for minimization (or the sign-flipped analogue
for maximization),
\begin{align}
\mathrm{PI}(\bz)
&=
\Phi\bigl(\gamma_{\mathrm{imp}}(\bz)\bigr),
\label{eq:pi}\\
\mathrm{EI}(\bz)
&=
\sigma(\bz)\,
\bigl[
\gamma_{\mathrm{imp}}(\bz)\,\Phi\bigl(\gamma_{\mathrm{imp}}(\bz)\bigr)
+
\phi\bigl(\gamma_{\mathrm{imp}}(\bz)\bigr)
\bigr]
\label{eq:ei}\\
&=
\mathbb{P}(f(\bz)<f_*)
\notag\\
&\qquad\cdot
\mathbb{E}\bigl[f_*-f(\bz)\mid f(\bz)<f_*\bigr].
\notag
\end{align}
Eq.~\eqref{eq:ei} is again an unconditional truncated-Gaussian moment:
EI factors as ``probability of improvement $\times$ conditional magnitude of
improvement,'' exactly as
Eq.~\eqref{eq:factorization} factors unconditional expected descent
\citep{jones1998efficient,mockus1978application}.
The directional analogues are therefore PI~$\leftrightarrow$~\mpd
(probability only), conditional improvement~$\leftrightarrow$~$\mathcal{P}(\bv)$
(magnitude given the event), and
EI~$\leftrightarrow$~unconditional expected descent
(probability $\times$ magnitude).

Progress-aware exploitation is thus closer to the \emph{conditional}
factor of EI than to EI itself: it deliberately discards the
probability weight $\Phi(-\gamma_\bv)$ and retains only the conditional
magnitude. 
Two distinctions remain important:
\begin{enumerate}[leftmargin=*,itemsep=2pt]
\item \textbf{Domain of the random variable.}
PI/EI truncate the posterior of $f(\bz)$ about a baseline $f_*$;
$\mathcal{P}$ truncates the posterior of a \emph{directional derivative}
about zero. The shared algebra is truncated-Gaussian, but the optimized
quantity is different.
\item \textbf{Role in the algorithm.}
PI/EI are typically used as global acquisition functions over $\calX$.
Progress-aware exploitation is a local direction-selection criterion
applied after gradient refinement has already concentrated
$(\bmu_\bx,\bSigma_\bx)$. Theorem~\ref{thm:progress_convergence} then shows that, as
$\Tr(\bSigma_\bx)\to 0$ with consistent mean, $\mathcal{P}$ recovers
normalized steepest descent, a limit that has no direct PI/EI analogue.
\end{enumerate}

In short, \mpd stands to progress-aware exploitation as PI stands to the
conditional half of EI, while unconditional expected descent stands to
$\mathcal{P}$ as EI stands to that same conditional half. \method
optimizes the conditional magnitude on directional derivatives, preceded by
closed-form gradient refinement.

\section{Gradient of the Progress-Aware Score}
\label{app:grad_progress}

To solve
\begin{equation}
\label{eq:max_progress_supp}
\bv^* = \argmax_{\|\bv\|=1} \mathcal{P}(\bv)
\end{equation}
via projected gradient ascent \citep{calamai1987projected}, we derive the analytic gradient of
$J(\bv)=\mathcal{P}(\bv)=\sigma_\bv\bigl[R(\gamma_\bv)-\gamma_\bv\bigr]$
with respect to $\bv$, where $R(\gamma):=\phi(\gamma)/\Phi(-\gamma)$ is the
inverse Mills ratio and $\gamma_\bv=\mu_\bv/\sigma_\bv$.

Recall $\mu_\bv=\bv^\top\bmu_\bx$,
$\sigma_\bv=\sqrt{\bv^\top\bSigma_\bx\bv}$,
and $\gamma_\bv=\mu_\bv/\sigma_\bv$. Their partial derivatives are
\begin{align}
\frac{\partial\sigma_\bv}{\partial\bv}
  &= \frac{\bSigma_\bx\bv}{\sigma_\bv}, \\
\frac{\partial\gamma_\bv}{\partial\bv}
  &= \frac{\bmu_\bx\sigma_\bv - \mu_\bv(\bSigma_\bx\bv/\sigma_\bv)}{\sigma_\bv^2}
   = \frac{\sigma_\bv\bmu_\bx - \gamma_\bv\bSigma_\bx\bv}{\sigma_\bv^2}.
\end{align}

Using the identity $R'(\gamma)=-R(\gamma)\bigl[\gamma+R(\gamma)\bigr]$
\cite{grimmett2020probability}, the chain rule gives
\begin{align}
\nabla_\bv J(\bv)
&= \frac{\partial\sigma_\bv}{\partial\bv}\bigl[R(\gamma_\bv)-\gamma_\bv\bigr]
 + \sigma_\bv\frac{\partial}{\partial\bv}\bigl[R(\gamma_\bv)-\gamma_\bv\bigr]
\notag\\
&= \frac{\bSigma_\bx\bv}{\sigma_\bv}\bigl[R(\gamma_\bv)-\gamma_\bv\bigr]
\notag\\
&\quad + \sigma_\bv
  \Bigl\{R'(\gamma_\bv)-1\Bigr\}
  \frac{\sigma_\bv\bmu_\bx-\gamma_\bv\bSigma_\bx\bv}{\sigma_\bv^2}
\notag\\
&= \frac{\bSigma_\bx\bv}{\sigma_\bv}\bigl[R(\gamma_\bv)-\gamma_\bv\bigr]
\notag\\
&\quad - \frac{\sigma_\bv\bmu_\bx-\gamma_\bv\bSigma_\bx\bv}{\sigma_\bv^2}
\notag\\
&\qquad\times
  \Bigl\{R(\gamma_\bv)\bigl[\gamma_\bv+R(\gamma_\bv)\bigr]+1\Bigr\}.
\label{eq:grad_progress}
\end{align}

\textbf{Projected gradient ascent.}
At each iteration update
$\bv' \leftarrow \bv + \alpha\nabla_\bv J(\bv)$
and project back to the unit sphere:
$\bv \leftarrow \bv'/\|\bv'\|$,
where $\alpha>0$ is the step size.
We use $10$ random restarts with warm-starting from the previous iteration's
best direction, and return the unit vector achieving the highest
$\mathcal{P}$ value.
The step size is set by Armijo backtracking line search with initial value
$\alpha_0=0.1$ and reduction factor $0.5$.

\section{Optimization and Computational Cost}
\label{app:optimization}

The sphere-constrained maximization of the progress-aware score
Eq.~\eqref{eq:max_progress_supp} has no general closed form.
We optimize $\mathcal{P}(\bv)$ with projected gradient ascent, $10$ random
restarts, and warm-starting from the previous iterate
(Section~\ref{app:grad_progress}).
Gradient refinement is deterministic once a query location is chosen:
$\alpha_{\mathrm{ref}}(\bz)=\|\nabla_\bx k_\calD(\bx,\bz)\|^2/(k_\calD(\bz,\bz)+\sigma^2)$
requires no Monte Carlo integration.

\paragraph{Maximizing $\alpha_{\mathrm{ref}}$ in the reported experiments.}
In all reported runs, we maximize $\alpha_{\mathrm{ref}}(\bz)$ by evaluating the closed-form expression on $N_{\mathrm{cand}}=100$ candidates sampled uniformly from the feasible set $\calX$, and selecting the highest-scoring candidate.
Each $\alpha_{\mathrm{ref}}$ evaluation is inexpensive relative to the
subsequent GP hyperparameter refit and black-box query, so we use this
lightweight random scoring rather than continuous multi-start optimization
over $\bz$.
Continuous or local multi-start maximization of $\alpha_{\mathrm{ref}}$, for
example, L-BFGS initialized from the top random candidates in a neighborhood
of $\bx$, remains a drop-in upgrade and was not used for the main-paper
numbers.

Each outer iteration requires
$(\tau_{\mathrm{explore}}+\tau_{\mathrm{exploit}})$ function evaluations in
the absence of early stopping. A naive exact-GP refit after an observation
costs $\mathcal{O}(|\calD|^3)$, although rank-one updates can reduce this
cost. In practice we retrain length scales and output scale by maximizing the
marginal likelihood after each observation.

\end{document}